\documentclass[12pt]{article}
\usepackage{graphicx} 
\usepackage{multirow}
\usepackage{subcaption}
\usepackage{listings}
\usepackage{amsmath}
\usepackage{amsfonts}
\usepackage{amscd}
\usepackage{amssymb}
\usepackage{natbib}
\usepackage{url}
\usepackage{bbm}
\usepackage{bm}
\usepackage{array}
\usepackage{geometry}
\usepackage{setspace}
\usepackage{hyperref}
\usepackage{amsmath}
\usepackage{xcolor}
\usepackage{mathtools}
\usepackage{etoc}
\usepackage{booktabs}
\usepackage{tikz}
\usetikzlibrary{positioning}

\hypersetup{
    colorlinks=true,
    linkcolor={rgb,255:red,0;green,0;blue,139},
    citecolor={rgb,255:red,0;green,0;blue,139},
    urlcolor={rgb,255:red,0;green,0;blue,139}
}

\def\m{\mathcal}
\def\mb{\mathbb}

\def\ind{\mathbbm{1}}

\newcommand{\mnorm}[1]{{\vert\kern-0.25ex\vert\kern-0.25ex\vert #1 
    \vert\kern-0.25ex\vert\kern-0.25ex\vert}}
\newcommand{\bmnorm}[1]{{\big\vert\kern-0.25ex\big\vert\kern-0.25ex\big\vert #1 
    \big\vert\kern-0.25ex\big\vert\kern-0.25ex\big\vert}}
\newcommand{\Bmnorm}[1]{{\Big\vert\kern-0.25ex\Big\vert\kern-0.25ex\Big\vert #1 
    \Big\vert\kern-0.25ex\Big\vert\kern-0.25ex\Big\vert}}
\newcommand{\bbmnorm}[1]{{\bigg\vert\kern-0.25ex\bigg\vert\kern-0.25ex\bigg\vert #1 
    \bigg\vert\kern-0.25ex\bigg\vert\kern-0.25ex\bigg\vert}}
\newcommand{\BBmnorm}[1]{{\Bigg\vert\kern-0.25ex\Bigg\vert\kern-0.25ex\Bigg\vert #1 
    \Bigg\vert\kern-0.25ex\Bigg\vert\kern-0.25ex\Bigg\vert}}

\newtheorem{theorem}{Theorem}

\newtheorem{lemma}[theorem]{Lemma}

\newtheorem{remark}{Remark}[theorem]

\newtheorem{definition}[theorem]{Definition}
\newtheorem{assumption}[theorem]{Assumption}

\newenvironment{proof}{{\bf Proof:}}{\hfill\rule{2mm}{2mm}}

\usepackage{setspace}
\usepackage{hyperref}

\usepackage{amsmath}
\usepackage{amssymb}
\usepackage{mathtools}
\usepackage{bbm}
\usepackage{bm}
\usepackage{tikz}
\usetikzlibrary{positioning}

\makeatletter

\let\c@remark\undefined

\newcounter{remark}
\renewcommand{\theremark}{\arabic{remark}}

\newenvironment{remark}
  {\par\noindent\textbf{Remark~\refstepcounter{remark}\theremark}\quad\itshape}
  {\par\addvspace{\baselineskip}}
\makeatother

\def\m{\mathcal}
\def\mb{\mathbb}

\def\ind{\mathbbm{1}}

\begin{document}
\makeatletter
\let\oldaddcontentsline\addcontentsline
\def\addcontentsline#1#2#3{} 
\makeatother

\def\spacingset#1{\renewcommand{\baselinestretch}%
{#1}\small\normalsize} \spacingset{1}


\title{PAC-Bayes Bounds for Gibbs Posteriors via Singular Learning Theory}
\date{}
\author{Chenyang Wang\thanks{cw80@illinois.edu}\, \textsuperscript{1}, Yun Yang\thanks{yy84@umd.edu}\, \textsuperscript{2}\\
  \textsuperscript{1}Department of Statistics, University of Illinois Urbana-Champaign\\
  \textsuperscript{2}Department of Mathematics, University of Maryland, College Park}
\maketitle

\setstretch{1}

\begin{abstract}%
We derive explicit non-asymptotic PAC-Bayes generalization bounds for Gibbs posteriors, that is, data-dependent distributions over model parameters obtained by exponentially tilting a prior with the empirical risk. Unlike classical worst-case complexity bounds based on uniform laws of large numbers, which require explicit control of the model space in terms of metric entropy (integrals), our analysis yields posterior-averaged risk bounds that can be applied to overparameterized models and adapt to the data structure and the intrinsic model complexity. The bound involves a marginal-type integral over the parameter space, which we analyze using tools from singular learning theory to obtain explicit and practically meaningful characterizations of the posterior risk. Applications to low-rank matrix completion and ReLU neural network regression and classification show that the resulting bounds are analytically tractable and substantially tighter than classical complexity-based bounds. Our results highlight the potential of PAC-Bayes analysis for precise finite-sample generalization guarantees in modern overparameterized and singular models.
\end{abstract}

\noindent
{\it Keywords:}
  Empirical risk minimization, Generalization bound, Gibbs posterior, Overparameterized model, PAC-Bayes, Singular model%

\section{Introduction}

Understanding and controlling the generalization error of learning algorithms remains a central goal in statistical learning theory \citep{neyshabur2017exploring, zhang2021understanding}. Classical bounds based on uniform laws of large numbers, such as those involving VC dimension, Rademacher complexity, or metric entropy control of the model class \citep{golowich2018size, yin2019rademacher, bartlett2019nearly}, provide \emph{worst-case} guarantees on the excess risk of empirical minimizers. While these tools have offered valuable insights, they often result in loose or vacuous bounds in modern overparameterized settings, especially when the model space is high-dimensional or weakly regularized. Indeed, recent theoretical work has shown that worst-case, uniform-convergence-based analyses may be fundamentally insufficient to explain the generalization behavior observed in deep learning \citep{nagarajan2019uniform, lotfi2022pac}, where models with far more parameters than data points can still achieve remarkable performance.

An alternative and increasingly prominent line of work uses Bayesian and PAC-Bayes frameworks \citep{smith2018bayesian,wilson2020bayesian,rivasplata2020pac} for \emph{averaged-case} analysis, which replace pointwise guarantees with risk bounds averaged over posterior distributions. These bounds have several practical advantages: they are often more data-dependent, avoid explicit covering number arguments \citep{alquier2021user}, and can provide tighter control of generalization error \citep{zhou2019non,lotfi2022pac,lotfi2024non}. In particular, Bayes-type bounds adapt to model complexity by balancing empirical fit and prior regularity through posterior averaging, without requiring control of the complexity of the entire model class.
In this work, we focus on the \emph{Gibbs posterior} \citep{zhang2006information, hong2020model, bhattacharya2022gibbs}, which is a loss-driven posterior distribution that generalizes the traditional Bayesian update by replacing the negative log-likelihood with an empirical risk function induced by a loss. This formulation allows for greater flexibility in modeling and improved robustness to model misspecification, while still enabling principled uncertainty quantification. Recent advances have established PAC-Bayes bounds for Gibbs posteriors under Bernstein-type conditions \citep{alquier2021user}, and corresponding posterior contraction rates have been derived by \citet{syring2023gibbs}, enabling finite-sample generalization guarantees even in the absence of an explicit likelihood structure.

Building on this line of work, we derive an explicit finite-sample PAC-Bayes generalization bound for Gibbs posteriors. In contrast to standard PAC-Bayes inequalities, which separately balance the empirical risk average and a complexity term expressed through the Kullback--Leibler (KL) divergence, our analysis combines these two components into a single bound involving a marginal-like integral over the parameter space. This quantity plays a role analogous to a partition function (or statistical evidence) in Bayesian inference. Through this integration effect, Bayesian methods exhibit intrinsic adaptivity to model complexity, providing a Bayesian counterpart to generalization control beyond worst-case complexity measures derived from uniform laws of large numbers. In regular models, this integral admits a classical Laplace approximation~\citep{geisser1990validity}, recovering the Bayesian Information Criterion \citep[BIC,][]{schwarz1978estimating}, which is known to balance goodness of fit and model complexity. Consequently, our formulation reveals that selecting a model that optimally trades off data fit and complexity is equivalent to selecting the model that achieves the tightest PAC-Bayes generalization bound.

To analyze this integral, we use tools from singular learning theory developed by \citet{watanabe2009algebraic}, which builds on \emph{algebraic geometry} to study the asymptotics of marginal likelihoods in non-regular models. In regular models, the leading term in the expansion of the marginal likelihood is $d/2$, where $d$ is the number of parameters. In singular models, the leading term is the \emph{real log canonical threshold} (RLCT), denoted by $\lambda$. In highly non-identifiable or overparameterized models, $\lambda$ can be much smaller than $d/2$ and may take fractional values rather than half-integers. This reflects the geometry of the parameterization: the set of minimizers of the population risk often forms a manifold with many flat directions along which the risk is constant. The RLCT effectively counts only the \emph{non-flat} directions, since parameters along flat directions do not contribute to the asymptotic expansion of the associated singular integral, resulting in sharper generalization bounds.

However, most existing results in singular learning theory~\citep[e.g.][]{watanabe2009algebraic} are asymptotic in nature and focus on fixed-dimension, increasing-sample-size regimes, which are not directly applicable to analyzing overparameterized models. They also typically rely on asymptotic expansions under abstract regularity conditions that are often difficult to verify in practice.
In contrast, we develop a new \emph{non-asymptotic} framework in which a sub-exponential loss assumption enables a two-sided PAC-Bayes control between the empirical and population risks. This bypasses the need for explicit control of metric entropy (integrals) or any worst-case uniform convergence arguments, and leads directly to fast-rate PAC-Bayes bounds that depend solely on the RLCT~$\lambda$ determined by the model and the loss function. This approach avoids many technical conditions required in asymptotic analyses and yields sharper finite-sample bounds on the expected posterior risk.

We illustrate our theory through two canonical examples: low-rank matrix completion \citep{alquier2013bayesian,alquier2014bayesian,yuchi2023bayesian} and deep ReLU neural networks \citep{nagayasu2023bayesian} under squared error loss, as well as classification with logistic loss \citep{chatterji2021does,zhang2024classification}. In particular, our analysis of the RLCT for matrix completion using blow-up techniques from algebraic geometry also provides insight into the geometry of the model, revealing how the model configurations can be divided into different regimes corresponding to distinct singularity patterns, each associated with a different RLCT.
In all cases, the resulting singular integrals admit either closed-form analyses or well-controlled analytical upper bounds, yielding explicit and interpretable generalization bounds that remain tight even in highly non-regular settings. This illustrates the potential of our framework for analyzing complex and overparameterized models, where traditional techniques may be less effective.

\section{Preliminaries and Problem Formulation}
In this section, we set up the notation, review empirical risk minimization (ERM), Gibbs posteriors, PAC-Bayes bounds, and singular learning theory, and reveal the connections among them that motivate our main theoretical results.

\subsection{Empirical risk minimization}\label{sec:ERM}
Let $Z=(X,Y)\sim\mb{P}$ be a generic data point in $\mathcal{X}\times\mathcal{Y}$, and let $\Theta$ be a compact parameter space. We consider learning problems specified by a loss function $\ell_\theta(Z):\mathcal{X}\times\mathcal{Y}\to\mb{R}$, which quantifies the discrepancy between the prediction under parameter $\theta$ and the observed response. Examples include the 0--1 loss $\ell_\theta(z)=\ind\{y\neq f_\theta(x)\}$ for classification, the squared error loss $\ell_\theta(z)=(y-f_\theta(x))^2$ for regression, the check loss $\ell_\theta(z)=\rho_\tau(y-f_\theta(x))$ for quantile regression, and the log-loss $\ell_\theta(z)=-\log p(y\mid x;\theta)$ used in likelihood-based or generative modeling.
Define the population risk $R(\theta)=\mb{E}[\ell_\theta(Z)]$ and the set of population minimizers $\Theta_0 := \arg\min_{\theta\in\Theta} R(\theta)$.
In non-identifiable or overparameterized models, $\Theta_0$ may be non-singleton. Throughout, we allow this possibility and require all assumptions to hold uniformly over $\theta^\star\in\Theta_0$; unless otherwise stated, $\theta^\star$ refers to an arbitrary element of this set. This non-uniqueness is inessential, as the generalization bounds developed in this paper are governed by the local algebraic structure of neighborhoods of $\Theta_0$, rather than by properties of any particular minimizer, all of which lead to the same hypothesis $f_{\theta^\star}(x)$. Although this work assumes that the hypothesis $f_\theta(x)$ is indexed by a finite-dimensional parameter $\theta$, the results extend to settings where $\theta$ itself is the hypothesis and may be infinite-dimensional, as in nonparametric regression with function-valued parameters, such as in Section~\ref{section 3} where we may identify $\theta$ with the regression function $f$.

Given $n$ i.i.d.\ observations $Z^n=(Z_1,\dots,Z_n)$, a natural point estimator of $\theta$ is
\[
\widehat{\theta}_n \in \arg\min_{\theta \in \Theta} \Big\{\,R_n(\theta):\,=\frac{1}{n}\sum_{i=1}^n \ell_\theta(Z_i)\,\Big\},
\]
where $R_n(\cdot)$ is called the empirical risk function. This estimation
procedure is known as \emph{empirical risk minimization} (ERM) \textcolor{black}{\citep{vapnik1999overview, zhang2017mixup, donini2018empirical}}.
Classical theory \textcolor{black}{\citep{shalev2014understanding}} for ERM shows that, under suitable
concentration and local strong convexity conditions, the excess risk of the empirical
minimizer can be controlled by a complexity measure of the loss class. A representative bound takes the form: with high probability,
\[
R(\widehat\theta_n)-R(\theta^\star)
\;\lesssim\;
\mathfrak{C}_n(\{\ell_\theta:\theta\in\Theta\}),
\]
where $\mathfrak{C}_n(\mathcal F)$ denotes a statistical complexity term
associated with a generic function class $\mathcal F$. Common choices include
the (localized) Rademacher complexity or Gaussian complexity \textcolor{black}{\citep{yin2019rademacher, 
truong2022rademacher}}, which quantify how well a function class can fit random labels or noise.

A standard way to upper bound such complexities is through metric entropy
integrals \textcolor{black}{\citep{golowich2018size}}, for example
$\Psi(r)=\int_0^r \sqrt{\log \mathcal{N}(\varepsilon,\Theta,\|\cdot\|)}\, d\varepsilon$,
under a suitable norm $\|\cdot\|$. Here,
$\mathcal{N}(\varepsilon,\Theta,\|\cdot\|)$ denotes the $\varepsilon$-covering
number, defined as the smallest number of radius-$\varepsilon$ balls needed
to cover the parameter space $\Theta$, and
$\log \mathcal{N}(\varepsilon,\Theta,\|\cdot\|)$ is called the covering
entropy. In regular finite-dimensional parametric models, this entropy typically
scales as $\log \mathcal{N}(\varepsilon,\Theta,\|\cdot\|)\asymp d\log(1/\varepsilon)$,
where $d$ is the ambient parameter dimension. However, in overparameterized
or singular models, $d$ can be a misleading proxy for statistical complexity,
since multiple parameter values may correspond to the same hypothesis,
making the intrinsic complexity substantially lower than that implied by
parameter counting.
Moreover, such entropy bounds are inherently worst-case: they rely on uniform
control of the entire parameter space and depend only on its global geometry,
rather than on the structure of the low-risk region actually explored by the
estimator. As a result, entropy-based bounds may become loose or even
vacuous, failing to reflect the effective, data-adaptive geometry of the
model that governs generalization in overparameterized and singular settings.
Localized complexity analyses \textcolor{black}{\citep{koltchinskii2006local, liang2015learning,
wainwright2019high}} can partially mitigate this issue by replacing global entropy with complexities evaluated near the empirical minimizer, thereby improving the sample-size dependence. Nevertheless, these bounds typically still require $d \ll n$; otherwise the bound does not vanish as
$n$ increases.

\subsection{Loss-based Gibbs posterior}
The uniform generalization guarantees of ERM provide
worst-case control over the entire parameter space and are typically
formulated for point estimators. In overparameterized and singular models,
however, generalization is governed by the geometry of the low-risk region
rather than the global size of the parameter space.
Bayesian methods offer a different perspective. Instead of producing a
single estimator, they define a data-dependent probability measure
$\Pi_n$ over the parameter space, leading to the averaged hypothesis
$\int_{\Theta} f_\theta(x)\, d\Pi_n(\theta)$ obtained by averaging over
parameters. This average-case viewpoint naturally adapts to the intrinsic
structure of the low-risk region and is therefore particularly advantageous
in complex models.

The standard Bayesian framework, however, relies on a correctly specified
likelihood model $p(y\,|\,x,\theta)$ to define the posterior. Under model
misspecification, classical Bayesian inference may produce inconsistent
estimates and underestimate uncertainty. The Gibbs posterior addresses this
issue by replacing the negative log-likelihood with a general loss function
\citep{zhang2006information,grunwald2012safe,bissiri2016general}, while
retaining the same exponential weighting structure. In this way, inference
no longer depends on specifying a full probabilistic model; parameters are
updated directly according to how well they minimize the observed loss,
which remains well-defined even when the likelihood is misspecified.
The Gibbs posterior, defined below, can be viewed as a Bayesian counterpart of frequentist ERM: posterior mass concentrates on parameters with small empirical loss, while the prior plays the role of a regularizer.
\begin{definition}[Gibbs posterior]
Given a prior distribution $\varphi(d\theta)$ on $\Theta$ and a learning rate
$\omega>0$, the Gibbs posterior is defined as
\begin{align}
\Pi_n(d\theta)
\propto
\exp\big(-\omega n R_n(\theta)\big)\,\varphi(d\theta),\quad \theta\in\Theta.
\label{eq:gibbs posterior}
\end{align}
\end{definition}
When $\ell_\theta(Z)=-\log p_\theta(Z)$, the empirical risk function reduces to the
negative log-likelihood function,
$R_n(\theta)=-\sum_{i=1}^n\log p_\theta(Z_i)$,
and the Gibbs posterior recovers the classical Bayesian fractional posterior
\citep{bhattacharya2019bayesian,alquier2020concentration,yang2020alpha}.

An equivalent characterization of the Gibbs posterior is given by the
variational formulation (or information risk minimization) 
\begin{align}
\Pi_n
=
\arg\min_{\rho \ll \varphi}
\left\{
\int_\Theta \,R_n(\theta)\,\rho(d\theta)
+
\frac{1}{\omega n}\,\mathrm{KL}(\rho\|\varphi)
\right\},
\label{eq:interp Pi n}
\end{align}
where the Kullback--Leibler divergence term $\mathrm{KL}(\rho\|\varphi)$ acts as a complexity penalty relative to the prior. This formulation shows that Gibbs inference can be viewed as a distributional version of penalized estimation over the space of all
probability measures on $\Theta$.

\subsection{PAC-Bayes generalization}

The PAC-Bayes framework provides non-asymptotic generalization guarantees
for randomized or averaged hypothesis represented by
data-dependent probability measures over the parameter space. In contrast
to classical uniform generalization bounds for point estimators, PAC-Bayes
theory controls the posterior-averaged risk through a trade-off between
empirical risk and a complexity penalty measured by the KL divergence
between a data-dependent measure and a prior. Therefore, it aligns naturally with the Gibbs posterior, which arises from the same empirical risk–KL trade-off as in~\eqref{eq:interp Pi n} and provides the appropriate theoretical framework for our analysis.

To formalize this perspective, we now introduce additional notation used to
introduce PAC-Bayes bounds. Throughout, we measure
performance relative to some $\theta^\star\in\Theta_0$ through the excess
loss $\ell(\theta,\theta^\star;Z):=\ell_\theta(Z)-\ell_{\theta^\star}(Z)$,
with corresponding empirical and population excess risks
\[
R_n(\theta,\theta^\star)
:=
\frac{1}{n}\sum_{i=1}^n \ell(\theta,\theta^\star;Z_i) \quad \mbox{and}
\quad
R(\theta,\theta^\star)
=
\mb{E}[\ell(\theta,\theta^\star;Z)]
\ge 0.
\]
Recall that $\varphi(\theta)$ denotes a prior over $\theta\in\Theta$, and let
$\rho\ll\varphi$ be any data-dependent probability measure on $\Theta$. 
For losses satisfying $\ell_\theta(z)\in[0,1]$, an early PAC-Bayes inequality
due to \citet{mcallester1999pac} states that, with probability at least
$1-\delta$ over i.i.d.\ samples $Z^n$, the following holds simultaneously
for all $\rho\ll\varphi$:
\begin{align*}
\mb{E}_{\theta \sim \rho} [R(\theta.\theta^\ast)]
\;\le\;
\mb{E}_{\theta \sim \rho} [R_n(\theta,\theta^\ast)]
+
\sqrt{\frac{\mathrm{KL}(\rho \,\|\, \varphi) + \log(1/\delta) + \frac{5}{2}\log n + 8}{2n-1}}.
\end{align*}
This bound shows that, for bounded losses, the generalization gap decays
at rate $n^{-1/2}$.

Faster $n^{-1}$ rates can arise when a Bernstein-type condition holds near the optimal set $\Theta_0$, where the variance of the loss is bounded by its mean, that is,
$\mathrm{Var}\bigl(\ell(\theta,\theta^\star;Z)\bigr)\le C\,R(\theta,\theta^\star)$.
For example, \citet[Theorem~4.3]{alquier2021user} show that for bounded losses, under a Bernstein-type condition and a sufficiently small learning rate $\omega$, the following expectation-level PAC-Bayes inequality holds:
\begin{align}
\mb{E}_{Z^n} \bigg[\int_\Theta R(\theta, \theta^\star)\, \Pi_n(d\theta)\bigg]
\le
2\inf_{\rho\ll\varphi}
\left\{
\mb{E}_{Z^n} \bigg[\int_\Theta R_n(\theta,\theta^\star) \,\rho(d\theta)\bigg]
\;+\;
\frac{C\,\mathrm{KL}(\rho\|\varphi)}{n}
\right\},
\label{eq:classic pac bayes}
\end{align}
This result can extend to unbounded losses with sub-exponential tails and can also be strengthened to high-probability bounds; see Section~\ref{section 3}. Moreover, the right-hand side admits a marginal-type integral interpretation,
$-\log\int_\Theta \exp\!\big\{-\tfrac{1}{C}n\,R_n(\theta,\theta^\star)\big\}\,\varphi(d\theta)$,
which is well suited to analysis via singular learning theory. Finally, with a slightly larger leading constant, we can replace $R_n(\theta,\theta^\star)$ in the exponent by $R(\theta,\theta^\star)$ without invoking additional uniform law of large numbers arguments; details can be found in Remark~\ref{rem:R_n} and the sketched proof of Theorem~\ref{thm:main} in Appendix~\ref{append:proof thm5}.


\subsection{Singular learning theory and real log canonical threshold}\label{sec:SLT}

To analyze the marginal likelihood (also known as Bayesian evidence) as a
special case of marginal-type integrals arising in singular settings, where
the Fisher information matrix is non-invertible and the standard Laplace
approximation based on a local quadratic expansion is no longer applicable,
\citet{watanabe2009algebraic} developed a general theory for integrals of
the form
\begin{align}
Z(n)
:=
\int_\Theta
\exp\{-n\beta\,K_n(\theta,\theta^\star)\}\,
\varphi(d\theta),
\label{eq: Zn}
\end{align}
where $\beta>0$ is a fixed inverse temperature parameter (which can be identified with
our learning rate $\omega$), and $K_n(\theta,\theta^\star) := n^{-1} \sum_{i=1}^n \log\frac{p^\star(z_i)}{p(z_i\,|\,\theta)}$ denotes the empirical log-likelihood ratio. Its population counterpart $K(\theta,\theta^\star) := \mb{E}\bigl[K_n(\theta,\theta^\star)\bigr]= \mathrm{KL}\big(p^\star\|p(\cdot\|\,\theta)\big)$ coincides with the KL divergence between the true data generating distribution $p^\star$ and the distribution $p(\cdot\,|\,\theta)$ from the parametric model family $\{p(\cdot\,|\,\theta):\theta\in\Theta\}$. Singular learning theory \citep{watanabe2009algebraic} characterizes the
asymptotic behavior of \eqref{eq: Zn} through the \emph{real log canonical
threshold} (RLCT), denoted by $\lambda$, which captures the local geometric
complexity near $\theta^\star$. Although originally developed for
log-likelihood-based losses, this analytic framework applies more broadly to
general loss functions and therefore extends beyond traditional Bayesian
models. 

To make this characterization precise, we follow and adapt the framework
of \citet{watanabe2009algebraic} to the general loss setting by defining
the RLCT through the analytic structure of the following zeta function associated with our population excess risk $R(\theta,\theta^\star)$,
\[ 
\zeta(s) := \int_{\Theta} R(\theta,\theta^\star)^s \,\varphi(d\theta), \qquad s \in \mb{C}, 
\] 
which is well defined and holomorphic for $\mathrm{Re}(s) > 0$ (the real part of complex number $s$). It is known \textcolor{black}{\citep{igusa1975complex, veys1997zeta}} that $\zeta(s)$ admits a meromorphic continuation to the entire complex plane whose poles are real, negative, and rational. Let \begin{align} 
    \lambda := -\max\{\text{poles of } \zeta(s)\}\label{eq:def rlct}
\end{align} denote the absolute value of the largest (rightmost) pole. This definition is natural because the pole closest to the origin governs the leading term in the asymptotic expansion of the associated Laplace-type integrals. Under suitable analytic and identifiability conditions, \citet[Main Theorem~6.2]{watanabe2009algebraic} shows that the log marginal likelihood admits the asymptotic expansion 
\begin{equation}
    -\log Z(n) = \lambda \log n - (m-1)\log\log n + \m O_P(1), \quad \mbox{as } n\to\infty,
\label{eqn:SBIC}
\end{equation}
where the RLCT $\lambda$ determines the pre-constant in front of the leading $\log n$ term, while $m\geq 1$, denoting the multiplicity of the pole, only contributes to the secondary double-logarithmic correction term. 

The standard Bayesian information criterion (BIC) approximation to the
Bayesian evidence \citep{schwarz1978estimating} corresponds to the case
$(\lambda, m)=(d/2,1)$ in regular parametric models, where the Hessian
matrix of $R(\theta,\theta^\star)$ (or, in the likelihood setting,
$K(\theta,\theta^\star)$) is nonsingular at $\theta^\star$. This arises as a
special case of the general theory and can be derived using the classical
Laplace approximation. More generally, $\lambda$ depends only on the local analytic structure of $K(\theta,\theta^\star)$ in a neighborhood of its minimizer set, is
invariant under analytic reparameterizations of $\Theta$, and can take
fractional values beyond half-integers (c.f.~Theorem~\ref{thm:completion}). In particular, $\lambda$ provides an intrinsic measure of complexity for
general loss-based models, extending beyond the classical log-likelihood
framework and regular parametric settings. Computing the
RLCT typically relies on tools from algebraic geometry, most notably
resolution of singularities \textcolor{black}{\citep{aoyagi2005stochastic, aoyagi2005resolution, aoyagi2010bayesian}}, and can be technically involved for complex models and/or parameter spaces. Roughly speaking, one constructs a resolution map that
transforms $R(\theta,\theta^\star)$ into a normal crossing form, from which
the RLCT can be read off from the associated exponents.
Appendix~\ref{append: blow up} provides further discussion and examples.

In addition, the singular learning theory developed by
\citet{watanabe2009algebraic} is inherently asymptotic and often requires
verifying technical conditions after analytic transformations of the
parameter space. Moreover, connecting the empirical log-likelihood ratio
$K_n$ to its population counterpart $K$ relies on variants of uniform laws
of large numbers, which implicitly require metric entropy control and lead to the
probabilistic $\mathcal{O}_P(1)$ remainder term in the expansion
\eqref{eqn:SBIC}. This $\mathcal{O}_P(1)$ term is inherently asymptotic in
nature and does not provide explicit finite-sample (or high-probability) control. Moreover, these arguments are not directly applicable to highly
overparameterized models where $d$ can be substantially larger than $n$.
In contrast, our approach works directly with posterior averages of the
empirical excess risk $R_n$ and relates them to posterior averages of the
population excess risk $R$ through PAC-Bayes techniques that do not require
explicit control of metric entropy, before invoking any resolution of 
singularities arguments. This allows us to develop and extend a new
non-asymptotic framework for singular learning theory. We refer to Remark~\ref{rem:R_n}
and Appendix~\ref{append:proof thm5} for related technical details.

\section{Our Main Results} \label{section 3}
In this section, we present our main PAC-Bayes generalization result for
Gibbs posteriors using tools from singular learning theory. We begin by imposing the following moment generating function condition on the (random) loss, which can be verified in many standard models and plays a crucial role in our analysis.

\begin{assumption}(Bernstein-type moment condition) \label{assump:overall} There exists two global constants $L > 0$ and $b > 0$ such that, for all $|\omega| \le \frac{1}{2b}$, the moment generating function (MGF) satisfies
    \begin{align}
        \mb{E}_Z[\exp\{\omega(\ell(\theta,\theta^\star;Z) - R(\theta, \theta^\star))\}]
         \le \exp\left\{ \frac{\omega^2L \cdot R(\theta, \theta^\star)}{2} \right\}.\label{eq:bernstein}
    \end{align}
\end{assumption}
This condition is a Bernstein-type moment control: for each $\theta$, the
centered loss $\ell(\theta,\theta^\star;Z)-R(\theta,\theta^\star)$ is
sub-exponential with variance parameter $L$ proportional to
$R(\theta,\theta^\star)$ and scale parameter $b$. This is a mild condition. As an illustration, we verify this for the squared loss in standard regression with sub-Gaussian noise, and for the logistic loss in classification under a margin condition.

\paragraph{Regression using squared loss.} Specifically, consider the regression model
$Y=f^\star(X)+\varepsilon$ with $\mb{E}[\varepsilon]=0$, and predictors
$f$ satisfying $|f(x)|\le B_0$ for all $x$. In this example, we identify the hypothesis $f_\theta$ with $f$. Then
$|\Delta_f(x)|:=|f(x)-f^\star(x)|\le 2B_0$.
Under the squared loss, for $Z=(X,Y)$ the excess loss can be written as
\[
\ell_{\rm sqr}(f,f^\star;Z)
=
(Y-f(X))^2-(Y-f^\star(X))^2
=
\Delta_f(X)^2 - 2\varepsilon\,\Delta_f(X),
\]
and the corresponding population excess risk is
$R(f,f^\star)=\mb{E}_X[\Delta_f(X)^2]$.
We now verify that Assumption~\ref{assump:overall} holds in this setting.

\begin{lemma}[Bernstein condition for squared loss]\label{lemma:verify assump}
Suppose $\varepsilon$ is $\sigma^2$-sub-Gaussian, mearning that
$\mb{E}[\exp\{t\varepsilon\}] \le
\exp\{\sigma^2 t^2/2\}$ for all $t\in\mb{R}$.
Then Assumption~\ref{assump:overall} holds with parameters
\begin{align}
\frac{1}{2b}
=
\min\left\{\frac{3}{16B_0^2},\,\frac{1}{2\sigma^2}\right\}\quad\mbox{and}\quad
L
=
32B_0^2 + 4\sigma^2.
\label{eq:b and L}
\end{align}
\end{lemma}

\paragraph{Classification using logistic loss.}
We now consider the classification setting with the logistic loss, a smooth and convex surrogate of the $0$--$1$ loss that is widely used in practice \citep{lyugradient,chizat2020implicit,chatterji2021does}. Let $Y\in\{-1,+1\}$ be binary labels and define
\[
\ell_{\rm logit}(f(X),Y)=\log\bigl(1+\exp(-Y\,f(X))\bigr).
\]
Let $\eta(X)=\mathbb{P}(Y=1\mid X)$. It is easy to show \citep{zhang2004statistical} that the conditional population risk is uniquely minimized at the log odds ratio
$f^\star(X)=\log \eta(X)-\log(1-\eta(X))$.
As in the regression setting, we also identify the hypothesis $f_\theta$ with $f$, and assume the logits are uniformly bounded, $|f(X)|\le B_3$, so the excess loss is bounded. We also impose a margin condition \citep{tsybakov2004optimal}: $\tau\le \eta(X)\le 1-\tau$ almost surely for some $\tau\in(0,1/2)$. Under these assumptions, the population excess risk admits a local quadratic approximation around $f^\star$ and is equivalent to $\mathbb E\big[(f(X)-f^\star(X))^2\big]$ in a neighborhood of $\theta^\star$.

\begin{lemma}[Bernstein condition for logistic loss]\label{lem:logistic}
Under the bounded-logit and margin conditions above, there exists a
neighborhood (relative to the sup-norm) of $f^\ast$ such that, for all
$f$ in this neighborhood, the excess logistic loss $\ell_{\rm logit}(f,f^\ast;Z)=\ell_{\rm logit}(f(X),Y) - \ell_{\rm logit}(f^\ast(X),Y)$
satisfies Assumption~\ref{assump:overall} with parameters
\begin{align*}
b = \log(1+e^{B_3}) \quad\mbox{and}\quad
 L=\frac{8}{\tau(1-\tau)}.
\end{align*}
\end{lemma}

\begin{remark}
A key advantage of Assumption~\ref{assump:overall} is that it accommodates
a broad class of losses and noise distributions, with bounded and Gaussian
noise as special cases. In particular, it generalizes the setting
of Theorem~4.3 in \citet{alquier2021user}, where only bounded losses are
considered. Moreover, this condition yields two-sided sub-exponential concentration for the centered loss, as quantified by inequalities~\eqref{eq:sketch-mgf-plus} and~\eqref{eq:sketch-mgf-minus} in Appendix~\ref{append:proof thm5}.
This allows us to relate posterior averages
of the empirical risk and the population risk in both directions, without
resorting to worst-case uniform bounds over the parameter space. By
contrast, \citet{syring2023gibbs} impose a one-sided condition of the form
$\mb{E}_Z\!\left[\exp\{-\omega\,\ell(\theta,\theta^\star;Z)\}\right]
\le \exp\{-K\omega\,R(\theta,\theta^\star)\}$, which still requires a
separate analysis of the empirical risk to obtain asymptotic
concentration rates.
\end{remark}

Under Assumption~\ref{assump:overall}, the PAC-Bayes inequalities can be
combined with the RLCT expansion to yield an explicit high-probability risk
bound for the Gibbs posterior, as shown in the following theorem.

\begin{theorem}[PAC-Bayes risk bound]\label{thm:main}
Suppose Assumption~\ref{assump:overall} holds. Let $\bar\omega := \min\left\{\frac{2}{L},\,\frac{1}{2b}\right\}$. Then, for any $\delta\in(0,1)$ and any $\omega\in(0,\bar\omega)$, with probability at least $1-\delta$,
\begin{align}
    \int_\Theta R(\theta,\theta^\star)\, \Pi_n(d\theta)
\;\le\;\tfrac{2}{\big(1-\frac{\omega L}{2}\big)\omega n} \bigg[-\log\int_\Theta\exp\Bigl\{-\tfrac{(1+\frac{\omega L}{2})\,\omega n}{2}\,R(\theta,\theta^\star)\Bigr\}\,\varphi(d\theta) + \log\frac{2}{\delta}\bigg],\label{eq:pacbayes-integral}
\end{align}
Further analysis of the integral term yields
\begin{align}
    \int_\Theta R(\theta,\theta^\star)\, \Pi_n(d\theta)
\;\le\;\tfrac{2}{\big(1-\frac{\omega L}{2}\big)\omega n} \Big[\lambda \log n - (m-1)\log \log n
+ \log\frac{2}{\delta} + C_0(\varphi, L, \omega)\Big],\label{eq:explicit}
\end{align}
where $\lambda$ denotes RLCT of the model and $C_0(\varphi, L, \omega)$ is independent of $n$.
\end{theorem}

\setcounter{remark}{1}
\begin{remark}\label{rem:R_n}
In a standard PAC-Bayes analysis, one typically stops at an inequality (e.g.~inequality~\eqref{eq:classic pac bayes}) which bounds the posterior average of the population risk by the posterior average of the empirical risk. Choosing the measure $\rho$ on the right-hand side as the Gibbs posterior $\Pi_n$ then leads to an upper bound of the form
$-\log \int_\Theta \exp\{-\omega n R_n(\theta,\theta^\star)\}\,\varphi(d\theta)$,
which still depends on the empirical risk $R_n$. This dependence is
precisely what prevents a direct application of singular learning theory.
The key advantage of our two-sided control from Assumption~\ref{assump:overall} is that we obtain a
high-probability version of the reverse-direction inequality of
\eqref{eq:classic pac bayes}.
 Combining the two inequalities yields a bound of the form \eqref{eq:pacbayes-integral}, 
whose right-hand side depends only on the deterministic population risk $R(\theta,\theta^\star)$. By selecting an
appropriate ``population level" Gibbs-type posterior $\rho(d\theta)\propto \exp\Bigl\{-\frac{(1+\omega L/2)\,\omega n}{2}\,R(\theta,\theta^\star)\Bigr\}\,\varphi(d\theta)$, we then obtain a marginal-type integral whose exponent involves the population risk rather than the empirical risk, allowing it to fit directly into the RLCT-based singular learning theory framework while
avoiding the $\mathcal{O}_P(1)$ remainder term in the expansion
\eqref{eqn:SBIC} that arises from asymptotic arguments.
\end{remark}

\begin{remark} Since $m\ge 1$, the second term in \eqref{eq:explicit} is non-positive.
Therefore, if only an upper bound $\bar\lambda \ge \lambda$ is available,
the entire expression in the brackets can be bounded by
$\bar\lambda \log n + \mathcal{O}(1)$. In practice, there are various
approaches for deriving such upper bounds.
A universal upper bound for $\lambda$ is $d/2$
\citep[Theorem~7.2]{watanabe2009algebraic}, where $d$ is the parameter
dimension, corresponding to the classical BIC penalty. This suggests that
the BIC penalty can be overly conservative in overparameterized models,
as it may substantially overestimate the intrinsic model complexity
captured by the RLCT. More refined, model-specific upper bounds have been
derived in particular cases of singular models, for example,
for ReLU networks in \citet{nagayasu2023bayesian} and for multivariate
Gaussian mixture models in \citet{yamazaki2003singularities}.

\end{remark}

\begin{remark}
As Theorem~\ref{thm:main} shows, the RLCT precisely captures the intrinsic
complexity of the parameter space, rather than the often overly pessimistic
covering numbers or entropies used in traditional bounds such as those
reviewed in Section~\ref{sec:ERM}, which are typically tied to the ambient
dimension $d$.
This comparison also highlights the advantage of working with
average-type estimators and analyses, which lie at the core of Bayesian
inference. Such an approach removes the need for global uniform control
over the entire parameter space and expresses the effective complexity
through integrals of exponentiated risk functions over the parameter
space. In particular, Bayesian integration over parameters naturally
induces an Occam's razor effect: models whose low-risk regions are geometrically small are automatically penalized through the integration, while models with large or diffuse low-risk regions receive higher posterior weight. This implicit complexity regularization arises without any explicit
tuning and reflects the intrinsic geometry of the model.
Our theoretical results make this intuition precise by deriving explicit
risk bounds in terms of such integrals, where the integration itself
automatically accounts for this penalty effect. The RLCT quantifies this
intrinsic integration-based complexity in a precise way.
\end{remark}

\section{Applications to Concrete Examples}
In this section, we demonstrate the consequences of our general theory by
applying it to two machine learning problems involving overparameterized singular models: matrix completion and ReLU neural network regression and classification.

\subsection{Matrix completion}\label{subsec:matrix completion}
We first consider the matrix completion problem, where the goal is to recover an unknown low-rank matrix $M^\star \in \mb{R}^{d_1 \times d_2}$ with rank $r\le\min\{d_1,d_2\}$ from partial and noisy observations. Assume we have $n$ noisy observations of the form $Y_k = M^\star_{I_k J_k} + \varepsilon_k$, $k=1,\dots,n$,
where each index pair $(I_k, J_k)$ is drawn independently and uniformly from $[d_1]\times[d_2]$, and $\varepsilon_k$ denotes additive noise with zero mean.
Let $M \in \mb{R}^{d_1 \times d_2}$ be a candidate matrix. Since usually the true rank $r$ is assumed to be much smaller than $d_1$ and $d_2$, following standard practice in the literature \citep{candes2010matrix, recht2013parallel, alquier2014bayesian}, we parameterize the candidate matrix through a low-rank factorization $M = UV$, where $U \in \mb{R}^{d_1 \times H}$ and $V \in \mb{R}^{H \times d_2}$ for some latent dimension $H$ such that $r\le H\ll\min\{d_1,d_2\}$.

Our goal is to assess the closeness of $M$ to $M^\star$, typically measured in the matrix Frobenius norm. To align with this objective, we use the following excess loss function:
\begin{align*}
    \ell(M,M^\star;Z_k) = (Y_{k} - M_{I_k J_k})^2 - (Y_{k} - M^\star_{I_k J_k})^2
    = (M_{I_k J_k} - M^\star_{I_k J_k})^2 - 2\varepsilon_{k}(M_{i_k j_k} - M^\star_{i_k j_k}),
\end{align*}
and let $R_n(M, M^\star)$ be the corresponding empirical risk function. The population excess risk is then
\[
R(M, M^\star) := \mb{E}[\ell(M, M^\star;Z)] = \frac{1}{d_1 d_2} \| M - M^\star \|_F^2,
\]
which exactly captures the Frobenius norm error up to a scaling factor.

Assume that the entries of $M$ are uniformly bounded, $\|M\|_\infty \le B_1$, where $\|\cdot\|_\infty$ denotes the elementwise maximum norm on matrices; this can be enforced by choosing a prior $\varphi$ supported on such matrices. Assume also that the noise variables $(\varepsilon_k)_{k=1}^n$ are independent, centered, and $\sigma_1^2$-sub-Gaussian. Then Assumption~\ref{assump:overall} holds for any
centered sub-Gaussian noise $\varepsilon$, as verified in
Lemma~\ref{lemma:verify assump}.
Additionally, we assume that the prior $\varphi(M)$ is supported on a compact parameter space $\mathcal M\subset \big\{ M = UV:\, U \in \mb{R}^{d_1 \times H},\,V \in \mb{R}^{H \times d_2}\big\}$ and
\(
\inf_{M\in\mathcal M}\varphi(M) = \varphi_0 > 0.
\)
Let $P_0\in\mb{R}^{d_1\times d_1}$ and $Q_0\in\mb{R}^{d_2\times d_2}$ be two invertible matrices such that the ground-truth matrix admits the rank-normal form
\(
M^\star \;=\; P_0
\begin{pmatrix}
I_r & 0\\
0 & 0
\end{pmatrix}
Q_0 .
\)
We introduce the corresponding transformed variables
\[
U' := P_0^{-1}U, \qquad V' := VQ_0^{-1}, \qquad M' := P_0^{-1}MQ_0^{-1},
\]
so that $M' = U'V'$.
Under the above setting, we have the following explicit PAC-Bayes bound for the excess risk as follows, where a precise RLCT characterization is obtained using blow-up techniques from algebraic
geometry and is technically involved, requiring very careful constructions (c.f.~Appendix~\ref{app:matrix completion}).

\begin{theorem}[PAC-Bayes bound for matrix completion]\label{thm:completion} Under the above setting, suppose $H+r\le d_1+d_2$, then it holds with probability at least $1-\delta$ that, {\color{black} with $L = 32B_1^2 + 4\sigma_1^2$},
    \begin{align*}
        \frac{1}{d_1d_2}  \int_{\mathcal M} \| M - M^\star\|_F^2\;\Pi_n(d\theta)
        \le \frac{2}{\big(1-\frac{\omega L}{2}\big)\omega n} \Big(\lambda \log n - (m-1)\log \log n -\log\varphi_0+ \log\frac{2}{\delta} + C_1\Big),
    \end{align*}
    where $C_1 = \mathcal{O}\big\{
  H^2\,(d_1+d_2)\,
  \log(r d_1 d_2 (1+L))\big\}\;+\; C(\omega,P_0,Q_0)$ {\color{black} with $C(\omega,P_0,Q_0)$ depends only on $P_0,Q_0$ and $\omega$ that grows at most linearly in $\omega$}. The corresponding RLCT $(\lambda,m)$ is given by
\[
(\lambda,m)
=
\begin{cases}
\Big(\dfrac{Hd_2 - Hr + d_1 r}{2},\,1\Big),
& H \le d_1 - d_2 + r,\\[0.4em]
\Big(\dfrac{Hd_1 - Hr + d_2 r}{2},\,1\Big),
& H \le d_2 - d_1 + r,\\[0.4em]
\Big(\dfrac{Hd_2 - Hr + d_1 r}{2} - \dfrac{(H - d_1 + d_2 - r)^2}{8},\,1\Big),
& \substack{H \;\ge\; r + |d_2 - d_1|,\\ H + d_1 + d_2 + r \text{ even},}\\[0.6em]
\Big(\dfrac{Hd_2 - Hr + d_1 r}{2} - \dfrac{(H - d_1 + d_2 - r)^2 + 1}{8},\,2\Big),
& \substack{H \;\ge\; r + |d_2 - d_1|,\\ H + d_1 + d_2 + r \text{ odd}.}
\end{cases}
\]
\end{theorem}

\setcounter{remark}{4}
\begin{remark}
Note that $\lambda$ in the theorem is much smaller than half of the ambient dimension $d_1d_2/2$ and also smaller than half of the total number of parameters in $(U,V)$, namely $H(d_1+d_2)/2$. Instead, it depends on the unknown true rank $r$ of $M^\star$.
The gap between $\lambda$ and $H(d_1+d_2)/2$ becomes particularly large when $d_1$ and $d_2$ are highly unbalanced. For example, if $d_1 \gg d_2 \ge H \gg r$ and $d_2$ is of the same order as $H$, corresponding to the first RCLT regime, then $d_1d_2$ and $H(d_1+d_2)$ are both of order $\mathcal{O}(d_1d_2)$, whereas $\lambda$ is only of order $\mathcal{O}(d_2^2+d_1r)$, which is much smaller. In such settings, the intrinsic complexity is governed by the smaller matrix factor between $U$ and $V$, while the larger factor becomes redundant and highly nonidentifiable. This also explains why the explicit expression of $\lambda$ splits into different regimes depending on the magnitude of $|d_1-d_2|$. A practical example arises in recommender systems, where the number of users $d_1$ is very large while the number of rating levels or item categories $d_2$ remains relatively small.
Moreover, the dependence of the RLCT on the intrinsic dimension of the true parameter, together with its ability to take fractional values beyond half-integers as seen here, is a distinctive feature of singular models. In contrast, the classical BIC penalty does not adapt to low-dimensional structure in the truth. At the same time, this dependence makes data-driven estimation of $\lambda$ challenging, which has motivated extensions of BIC to singular models, known as singular BIC \citep{drton2017bayesian}, for singular model selection.
\end{remark}

\subsection{Regression and classification using ReLU neural networks}
We next consider a ReLU neural network model following the setup of
\citet{nagayasu2023bayesian}, where an upper bound on the RLCT is derived
under a likelihood-based formulation with Gaussian noise, for which the
negative log-likelihood coincides with the squared loss. Our analysis,
however, only requires that the loss function satisfies
Assumption~\ref{assump:overall}, which is sufficient to obtain the same
type of RLCT-based complexity control within the developed Gibbs posterior framework.

Specifically, we consider an $N$-layer fully connected ReLU neural network
with layer widths $(H_1,\dots,H_N)$, where $H_1$ and $H_N$ denote the input
and output dimensions. For an input $x\in\mb{R}^{H_1}$, the network
output is defined recursively by
\begin{align*}
f^{(1)}(x) = x \quad \mbox{and}\quad
f^{(k)}(W,b,x)
=
\sigma\bigl(W^{(k)} f^{(k-1)}(W,b,x) + b^{(k)}\bigr),
\qquad 2 \le k \le N,
\end{align*}
with final output $f^{(N)}(W,b,x)\in\mb{R}^{H_N}$ where $W=(W^{(k)})$ and $b=(b^{(k)})$. Here
$\sigma(t)=\max\{t,0\}$ is the ReLU activation applied componentwise,
$W^{(k)}\in\mb{R}^{H_k\times H_{k-1}}$ are the weight matrices, and
$b^{(k)}\in\mb{R}^{H_k}$ are bias vectors. We denote the full parameter
by $\theta=(W,b)$.

We assume that the data are generated according to a realizable but overparameterized model. Specifically, the response variable satisfies
\[
Y = f^{(N^\star)}(W_0,b_0,X) + \varepsilon,\qquad N^\star\le N
\]
where $(W_0,b_0)$ correspond to a ReLU network with widths $(H_1^\star,\dots,H_{N^\star}^\star)$ such that $H_1^\star=H_1$, $H_{N^\star}^\star=H_N$, $H_k^\star \le H_k$ for $(2\le k\le N^\star-1)$, and $H_{N^\star-1}^\star \le H_k$ for $(N^\star\le k\le N-1)$. When $H_N>1$, this corresponds to a multivariate response regression. Under these conditions, the ReLU network model with layer sizes
$H_1,\ldots,H_N$ is overparameterized yet well-specified, in the sense that
it contains the true data-generating mechanism. In particular, there exists
a (generally non-unique) parameter $\theta^\star=(W^\star,b^\star)$ such
that
$f^{(N)}(W^\star,b^\star,x)=f^{(N^\star)}(W_0,b_0,x)$ for all $x$
(for example, by expanding the matrices $W_0$ and vectors $b_0$ with zeros
to match the larger dimensions). We denote by $\Theta_0$ the set of all
such realizable parameters.

Consider the squared-loss-based Gibbs posterior, where given an observation $Z=(X,Y)$, the loss function is defined as $\ell(\theta;Z) = \big\|Y - f^{(N)}(\theta,X)\big\|_2^2$.
We assume that the prior density $\varphi(\theta)$ is strictly positive and
continuous on the parameter space $\Theta$, and that
$\|f^{(N)}(\theta,x)\|\le B_2$ for all admissible network parameters and
inputs $x$, which can be ensured by restricting $\Theta$ to a compact set.
Under this condition, Assumption~\ref{assump:overall} holds when
$\varepsilon$ is centered $\sigma_2^2$-sub-Gaussian noise, as verified in
Lemma~\ref{lemma:verify assump}.
Combining Theorem~\ref{thm:main} with the upper bound on the RLCT established in Theorem~3.1 of \citet{nagayasu2023bayesian}, we obtain the following excess risk bound.

\begin{theorem}[ReLU regression risk bound]\label{prop:relu}
Under the above setting, the Gibbs posterior satisfies the following excess risk bound with {\color{black} with $L = 32B_2^2 + 4\sigma_2^2$} and probability at least $1-\delta$:
\begin{align*}
     \int_{\Theta} \mb E_X\Big[\bigl\| f^{(N)}(\theta,X)-f^{(N^\star)}(\theta^\star,X)\bigr\|_2^2\Big]\,\Pi_n(d\theta) \le \frac{2}{\big(1-\frac{\omega L}{2}\big)\omega n} \Big[\bar\lambda_{\rm{ReLU}} \log n + \log\frac{2}{\delta} + C_2\Big],
\end{align*}
where $\bar\lambda_{\mathrm{ReLU}} = \frac{1}{2}\big(\sum_{k=2}^{N^\star}H_k^\star(H_{k-1}^\star + 1) \big),\label{eq:relu rlct}$
is an upper bound of the RLCT $\lambda_{\mathrm{ReLU}}$ for the ReLU network model, equal to one half of the number of free parameters in the minimal network, and $C_2$ is a constant independent of $n$.
\end{theorem}

This theorem shows that $\lambda_{\mathrm{ReLU}}$ is upper bounded by the minimal network that realizes the true regression function, regardless of the size of the neural network used.
A similar risk bound also applies to neural network classification ($H_N=1$) with logistic loss  via Lemma~\ref{lem:logistic}. In particular, if the true log-odds function $f^\star$ is realizable by the ReLU network under consideration, then the RLCT upper bound in \eqref{eq:relu rlct} continues to hold in this classification setting. Due to space constraints, we omit the concrete results and details.

\section{Discussion}
This work studies finite-sample generalization of Gibbs posteriors under
general loss functions and singular model structures. By specializing
PAC-Bayes bounds to the Gibbs posterior, we obtain a marginal-type
integral over the parameter space that can be analyzed using singular
learning theory. A Bernstein-type condition linking empirical and
population risks allows this analysis to proceed without requiring
explicit metric entropy control, leading to potentially sharp
generalization bounds for overparameterized and singular models.

In such settings, the generalization bounds are governed by the real log canonical threshold (RLCT), which provides a sharper characterization than existing approaches that identify and leverage low-dimensional structure in the population risk. For example, approaches based on restricted local convexity measure complexity using quantities such as the Gaussian width of directions with non-flat curvature and typically yield only coarse upper bounds, whereas the RLCT can be viewed as a more refined, continuous measure of complexity than simply counting dimensions.


\bibliographystyle{apalike}

\bibliography{reference}

\newpage
\appendix

\section{RLCT and Resolution of Singularities}\label{append: blow up}

This appendix provides a brief overview of the derivation of the RLCT and several key results from singular learning theory, with the aim of supporting the main text. For comprehensive treatment and rigorous proofs, we refer readers to \citet{watanabe2009algebraic} and other related works.

\subsection{Derivation of the RLCT and normal crossing form}

To make the definition of RLCT given in \eqref{eq:def rlct} operational, one typically invokes the resolution of singularities (see Theorem~2.3, Hironaka's theorem, in \citet{watanabe2009algebraic}). This result ensures that, in a neighborhood of $\Theta_0$, there exists a finite collection of local coordinate charts $\{U_\kappa\}_\kappa$, each diffeomorphic to $[0,1]^d$, together with analytic changes of variables $\theta = g_\kappa(u)$ defined on $U_\kappa$, such that the union of the images $g_\kappa(U_\kappa)$ covers the region of interest.

Under these local parameterizations, both the excess risk function and the associated Jacobian determinant admit a normal crossing form. Specifically, for each chart $\kappa$, we have
\begin{align*}
R(\theta,\theta^\star)
& =
R(g_\kappa(u))
=
S_\kappa(u)\, u_1^{2k_{\kappa,1}} u_2^{2k_{\kappa,2}} \cdots u_d^{2k_{\kappa,d}},\\
|g_\kappa'(u)| & =
b_\kappa(u)\, u_1^{h_{\kappa,1}} u_2^{h_{\kappa,2}} \cdots u_d^{h_{\kappa,d}},
\end{align*}
where $S_\kappa(u)$ and $b_\kappa(u)$ are analytic and non-vanishing functions on $U_\kappa$, and the integers $k_{\kappa,1},\ldots,k_{\kappa,d} \ge 0$ and $h_{\kappa,1},\ldots,h_{\kappa,d} \ge 0$ depend on the chosen local chart.

Under this normal crossing representation, the poles of the zeta function can be read off explicitly. For each local chart $\kappa$, the candidate poles are given by
\[
\lambda_{\kappa,j}=-\frac{h_{\kappa,j}+1}{2k_{\kappa,j}}, \qquad j=1,\ldots,d,
\]
with the convention that $(h_{\kappa,j}+1)/2k_{\kappa,j} = \infty$ when $k_{\kappa,j}=0$. Consequently, the RLCT and its order are obtained by taking the minimum over all local charts and coordinates and counting the number of directions achieving $\lambda$,
\[
\lambda
\;=\;
\min_{\kappa}\;\min_{1 \le j \le d} \frac{h_{\kappa,j}+1}{2k_{\kappa,j}},
\qquad
m
\;:=\;
\max_{\kappa}
\#\Bigl\{\, j \, :\, \frac{h_{\kappa,j}+1}{2k_{\kappa,j}} = \lambda \Bigr\}.
\]
It is important to emphasize that the integers $(k_{\kappa,j},h_{\kappa,j})$ do not define the RLCT themselves; rather, they provide a convenient local representation for computing the location and multiplicity of the dominant pole of the zeta function within a given coordinate chart. The invariance of the RLCT follows from the fact that it is defined intrinsically through the pole structure of $\zeta(s)$ and therefore does not depend on the particular choice of local charts or normal crossing representations.

\subsection{The blow-up technique}

While the resolution-of-singularities theorem guarantees the existence of a normal crossing representation, it does not directly explain how such a representation is constructed in practice. Geometrically, singularities of analytic functions often arise from the intersection of multiple manifolds, where different components of the zero set meet in a non-smooth manner. The blow-up technique provides a systematic way to resolve such singularities by replacing a singular point with a higher-dimensional geometric object that encodes the local directions along which the singular components intersect.

To illustrate the idea, consider the analytic function
\[
f(x,y,z) = (x^2+y^2)z^2,
\]
whose zero set consists of the union of the $z$-axis $\{x=y=0\}$ and the plane $\{z=0\}$, intersecting at the origin. The singularity at the origin arises precisely from this non-transversal intersection. The key idea of the blow-up is to replace the single point $(0,0,0)$ by a collection of local coordinate charts, each corresponding to a distinct direction along which the singular components separate.

Concretely, we introduce two local coordinate systems via the analytic changes of variables
\[
\begin{aligned}
&\text{Chart 1:}\qquad
x = x_1,\quad y = x_1 y_1,\quad z = z_1, \\
&\text{Chart 2:}\qquad
x = x_2 y_2,\quad y = y_2,\quad z = z_2.
\end{aligned}
\]
The corresponding Jacobian determinants are given by
\[
|g_1'(u)| = |x_1|,
\qquad
|g_2'(u)| = |y_2|.
\]

Under these transformations, the function $f$ becomes
\[
f(x,y,z)
=
x_1^2 z_1^2 (1+y_1^2)
\quad\text{in Chart~1,}
\qquad
f(x,y,z)
=
y_2^2 z_2^2 (1+x_2^2)
\quad\text{in Chart~2.}
\]
Since the analytic factors $(1+y_1^2)$ and $(1+x_2^2)$ are strictly positive and do not introduce any additional singularities, the function is already in normal crossing form on each chart.

We now read off the normal crossing exponents. In Chart~1, the monomial part is $x_1^{2} z_1^{2}$, with exponents $2k_{x_1}=2$, $2k_{z_1}=2$, and $2k_{y_1}=0$, while the Jacobian contributes $h_{x_1}=1$ and $h_{y_1}=h_{z_1}=0$. The corresponding candidate values are
\[
\frac{h_{x_1}+1}{2k_{x_1}}=\frac{2}{2}=1,
\qquad
\frac{h_{z_1}+1}{2k_{z_1}}=\frac{1}{2}.
\]
In Chart~2, the roles of $x$ and $y$ are symmetric, yielding the same candidate values.
Taking the minimum over all charts and coordinates, the real log canonical threshold of $f(x,y,z)=(x^2+y^2)z^2$ is therefore
\[
\lambda = \frac{1}{2},
\qquad
m = 1.
\]

This example illustrates how the blow-up technique separates intersecting components of the zero set into simpler pieces on which the singular structure becomes monomial. The RLCT is then determined by the most singular direction across all local charts, consistent with its definition as the dominant pole of the associated zeta function.

For a more detailed introduction to the blow-up technique and further examples, we refer the reader to Chapters~2 and~3 of \citet{watanabe2009algebraic}, as well as to \citet{aoyagi2005stochastic, aoyagi2005resolution, aoyagi2010bayesian}, which focus on specific classes of models and singularity structures.

\section{Proof of the Main Results}

\subsection{Proof of Lemma~\ref{lemma:verify assump}}
\noindent \begin{proof}
    Conditional on $X = x$, we have
\begin{align}
    & \mb{E}_{\varepsilon}\left[
        \exp\{\omega(\ell(f,f^\star;Z)-R(f,f^\star))\}\,\big|\,X=x
    \right]\nonumber\\
    = & \exp\{\omega(\Delta_f(x)^2 - \mb{E}_X[\Delta_f(X)^2])\}
       \mb{E}_{\varepsilon}\left[\exp\{-2\omega\varepsilon\Delta_f(x)\}\right].
       \label{eq:condition-on-X}
\end{align}
Since $\varepsilon$ is $\sigma^2$-sub-Gaussian, then $\mb{E}_{\varepsilon}\left[\exp\{-2\omega\varepsilon\Delta_f(x)\}\right] \le \exp\{2\sigma^2\omega^2\Delta_f(x)^2\}$ holds for any $\omega\in\mb{R}$. Substituting this bound into \eqref{eq:condition-on-X} and rearranging terms yields, the right-hand side is upper bounded by
\begin{align*}
    \exp\Big\{\big(\omega+2\sigma^2\omega^2\big)\big(\Delta_f(x)^2 - \mb{E}_X[\Delta_f(X)^2]\big) + 2\sigma^2\omega^2\mb{E}_X[\Delta_f(X)^2]\Big\}.
\end{align*}

To derive the MGF under the joint distribution of $Z$, we take the expectation of the above display with respect to $X$. 
Since $|\Delta_f(x)|\le 2B_0$ and hence
\[
|\Delta_f(x)^2| \le 4B_0^2,
\qquad
\bigl|\Delta_f(x)^2 - \mb{E}_X[\Delta_f(X)^2]\bigr| \le 4B_0^2.
\]
Moreover,
\[
\mathrm{Var}\bigl(\Delta_f(X)^2\bigr)
\le \mb{E}_X\big[\Delta_f(X)^4\big]
\le (2B_0)^2 R(f,f^\star).
\]

By the standard Bernstein MGF inequality for bounded random variables, we obtain
\begin{align*}
    \mb{E}_X \big[\exp\{t(\Delta_f^2 - \mb{E}_X[\Delta_f^2]) \}\big]
    \le \exp\bigg(\frac{t^2\,4B_0^2 R(f,f^\star)}{2\bigl(1-4tB_0^2/3\bigr)}\bigg)
\end{align*}
for all $t\in(0, 3/(4B_0^2))$. Set $t := \omega + 2\sigma^2\omega^2$, 
and choose $b$ as in \eqref{eq:b and L}.
Then, for any $|\omega|\le 1/(2b)$, we have $|t| \le |\omega| + 2\sigma^2|\omega|^2 = 2|\omega| \le 3/(8B_0^2)$, so that $1 - 4tB_0^2/3 \;\ge\; \frac{1}{2}$.
Consequently,
\begin{align*}
    \log\mb{E}_X \big[\exp\{t(\Delta_f^2 - \mb{E}_X[\Delta_f^2]) \}\big]\le 16B_0^2 \omega^2 R(f,f^\star),
\end{align*}
where we used $t^2\le (2|\omega|)^2 = 4\omega^2$.

Combining the bounds on the conditional and marginal expectations, we conclude that
\begin{align*}
    \log\mb{E}_Z \exp\{\omega(\ell(f,f^\star;Z)-R(f,f^\star))\} \le  2\sigma^2\omega^2 R(f,f^\star) + 16B_0^2 \omega^2 R(f,f^\star),
\end{align*}
for all $|\omega|\le 1/(2b)$. Therefore, Assumption~\ref{assump:overall} with $L =32B_0^2 + 4\sigma^2$.
\end{proof}

\subsection{Proof of Lemma~\ref{lem:logistic}}
\noindent \begin{proof}
    Since we assume that $|f_\theta(x)|\le B_3$ for any $\theta\in\Theta$ and $x$, the excess loss $\ell(f_\theta,f^\star;Z)$ is bounded by $\log(1+e^{B_3})$. Therefore, by a standard Bernstein-type MGF bound for bounded random variables, when $|\omega|< 3/b$,
\[
\ln\mb E_Z\exp\big\{\omega\big(\ell(f_\theta,f^\star;Z) - R(f_\theta,f^\star)\big)\big\} \le \frac{\omega^2 V}{2(1-\omega b/3)},
\]
where $V = \mathrm{Var}(\ell(f_\theta,f^\star;Z))$ and $b = \log(1+e^{B_3})$. To verify Assumption~\ref{assump:overall}, it remains to bound $V$ by a finite multiple of $R(f_\theta,f^\star)$ on a suitable neighborhood of $\theta^\star$.

Under the margin condition, the population risk admits a locally quadratic approximation around the optimal predictor $f^\star$. To see this, the conditional population risk can be written as a function of the scalar $f$:
\[
R(f;X)
:=\mb E_{Y\,|\, X}[\ell(f,Y)]
= \eta \log\bigl(1+e^{-f}\bigr)
+ (1-\eta)\log\bigl(1+e^{f}\bigr).
\]
Then the population excess risk $R(f,f^\star;X)=R(f;X)-R(f^\star;X)$ is smooth and strictly convex in $f$, with a unique minimizer given by $\log(\eta(X)) - \log (1-\eta(X))$. 
So derivatives of $R(\,\cdot\,,f^\star;X)$ coincide with those of $R(\,\cdot\,;X)$.

Differentiating $R(f;X)$ with respect to $f$ yields
\[
R'(f;X)
= \eta\Bigl(-\frac{1}{1+e^{f}}\Bigr)
+ (1-\eta)\Bigl(\frac{e^{f}}{1+e^{f}}\Bigr)
= \frac{e^{f}}{1+e^{f}}-\eta.
\]
Denoting the logistic function by $\sigma(f)=\frac{1}{1+e^{-f}}$, we obtain $R'(f;X)=\sigma(f)-\eta$.
Evaluating at the unique minimizer $f^\star$ gives
\[
R'(f^\star;X)=\sigma(f^\star) - \eta = \frac{\eta/(1-\eta)}{1+\eta/(1-\eta)} - \eta = 0.
\]
We observe that $\sigma(f^\star) = \eta$. Differentiating once more,
\[
R''(f;X)=\sigma'(f)
=\frac{e^{f}}{(1+e^{f})^2}
=\sigma(f)\bigl(1-\sigma(f)\bigr).
\]
Evaluating at $f=f^\star$ and using $\sigma(f^\star)=\eta$ yields
\[
R''(f^\star;X)
=R''(f^\star;X)
=\eta(X)\bigl(1-\eta(X)\bigr),
\]
so under the margin condition we have
\(\tau(1-\tau)\le R''(f^\star;X)\le 1/4\) uniformly in $X$.
Moreover, $R''(f;X)=\sigma(f)(1-\sigma(f))$ is continuous in $f$, and by the bounded-logit
assumption $|f|\le B_3$ its dependence on $f$ is uniformly continuous on $[-B_3,B_3]$.
Hence the second order derivative $R''(f;X)$ is uniformly continuous in $f$ and uniformly bounded away from $0$ in a neighborhood of $f^\star(X)$, uniformly over $X$.

Consequently, a second-order Taylor expansion of $R(f_\theta,f^\star;X)$ around $f^\star(X)$ implies that for any $\gamma>0$, there exists sufficiently small $\delta_\gamma>0$  such that whenever \(|f(X)-f^\star(X)| \le \delta_\gamma\), the conditional excess risk satisfies the two-sided bound
\begin{align*}
    \left| \frac{R(f_\theta,f^\star;X)}{(f_\theta(X)-f^\star(X))^2} - \tfrac12\,\eta(X)\big(1-\eta(X)\big)\right|\le \gamma.
\end{align*}
Choose $\gamma = \frac{1}{4}\tau(1-\tau)$. Using the margin assumption $\tau \le \eta(X)\le 1-\tau$ the fact $\eta(X)(1-\eta(X))\le 1/4$, this yields the uniform bounds
\begin{align}
    & R(f_\theta,f^\star;X)
\;\ge\; \tfrac14\,\tau(1-\tau)\,(f_\theta(X)-f^\star(X))^2,\label{eq:logit bound L2}\\
& R(f_\theta,f^\star;X)
\;\le\;
\Bigl(\tfrac18+\tfrac14\,\tau(1-\tau)\Bigr)\,(f_\theta(X)-f^\star(X))^2.\label{eq:L2 bound logit}
\end{align}
We now restrict the analysis to the local neighborhood \(\Theta_\mathrm{loc}:=\{\theta\in\Theta:\ \sup_X|f_\theta(X)-f^\star(X)|\le\delta_\gamma\}\).
Since restricting the domain of integration can only decrease the value of the integral, the negative logarithm of the restricted integral provides an upper bound for the original quantity.

We now complete the verification of Assumption~\ref{assump:overall} by bounding the variance $V$ in terms of the excess risk on the local neighborhood $\Theta_{\mathrm{loc}}$. Since the logistic loss is $1$-Lipschitz in its argument:
\[
|\ell(f_\theta(X),Y)-\ell(f^\star(X),Y)|
\le
|f_\theta(X)-f^\star(X)|.
\]
In addition, $\mathrm{Var}(\ell(f_\theta,f^\star;Z))\le \mb E_X\mb E_{Y\,|\,X}[\ell^2 (f_\theta,f^\star;Z)]$. Combining this with the lower bound in \eqref{eq:logit bound L2}, we obtain, for $\theta\in \Theta_\mathrm{loc}$,
\[
\mb{E}_{Y\,|\,X}\left[\ell^2 (f_\theta,f^\star;Z)\right]
\le
\frac{4}{\tau(1-\tau)}\,R(\theta,\theta^\star;X).
\]
Taking expectation with respect to $X$ on both sides and using that $1-\frac{\omega b}{3}\ge \tfrac12$ for all $|\omega|\le3/(2b)$, we verify Assumption~\ref{assump:overall} for all $|\omega|\le 1/(2b)$ with $b=\log(1+e^{B_3})$ and \(L=8/\big[\tau(1-\tau)\big]\), which completes the proof.

Moreover, the two-sided bounds \eqref{eq:logit bound L2} and \eqref{eq:L2 bound logit} imply that, on $\Theta_{\mathrm{loc}}$, the logistic-loss-based excess risk $R(f_\theta,f^\star)$ is equivalent, up to positive multiplicative constants, to the squared $L^2$ distance $\mb{E}_X(f_\theta(X)-f^\star(X))^2$, which provides the justification for the discussion above Lemma~\ref{lem:logistic}.
\end{proof}

\subsection{Proof of Theorem~\ref{thm:main}}\label{append:proof thm5}
\noindent \begin{proof}
We first present a proof sketch and then justify the stated inequalities in the following.

The starting point is Assumption~\ref{assump:overall}, which yields a two-sided sub-exponential control of the centered loss. In particular, by rearranging \eqref{eq:bernstein}, one obtains that with $\bar\omega$ defined in Theorem~\ref{thm:main}, for any $\omega\in(0,\bar\omega)$ and all $\theta\in\Theta$,
\begin{align}
    \mb{E}_Z\bigl[\exp\{-\omega\ell(\theta,\theta^\star; Z)\}\bigr]
    &\le
    \exp\bigl\{-(1-\tfrac{\omega L}{2})\,\omega\,R(\theta,\theta^\star)\bigr\},
    \label{eq:sketch-mgf-minus}
    \\
    \mb{E}_Z\bigl[\exp\{\omega\ell(\theta,\theta^\star; Z)\}\bigr]
    &\le
    \exp\bigl\{(1+\tfrac{\omega L}{2})\,\omega\,R(\theta,\theta^\star)\bigr\}.
    \label{eq:sketch-mgf-plus}
\end{align}
Applying the standard PAC-Bayes change-of-measure argument with \eqref{eq:sketch-mgf-minus} yields, for any probability measure $\rho \ll \varphi$ and any $\delta\in(0,1)$, a high-probability bound of the form
\begin{align}
    \int R(\theta,\theta^\star)\,\Pi_n(d\theta) \;\le\; \frac{1}{1-\omega L/2}\int R_n(\theta,\theta^\star)\,\rho(d\theta) + \frac{\mathrm{KL}(\rho\|\varphi)+\log(1/\delta)}{(1-\omega L/2)\,\omega n}.\label{eq:R bound by Rn}
\end{align}
A second application of the same argument, now using \eqref{eq:sketch-mgf-plus}, gives
\begin{align}
    \int R_n(\theta,\theta^\star)\,\rho(d\theta) \;\le\; \bigl(1+\tfrac{\omega L}{2}\bigr)\int R(\theta,\theta^\star)\,\rho(d\theta) + \frac{\mathrm{KL}(\rho\|\varphi)+\log(1/\delta)}{\omega n}.\label{eq:Rn bound by R}
\end{align}
Replacing $\delta$ by $\delta/2$ in both inequalities and combining them via a union bound yields, for any $\rho\ll\varphi$ and any $\delta\in(0,1)$, with probability at least $1-\delta$,
\begin{align}
    \int R(\theta,\theta^\star)\,\Pi_n(d\theta)
    \;\le\;
    \frac{1+\omega L/2}{1-\omega L/2}\int R(\theta,\theta^\star)\,\rho(d\theta)
    +
    \frac{2\bigl[\mathrm{KL}(\rho\|\varphi)+\log(2/\delta)\bigr]}{(1-\omega L/2)\,\omega n}. \label{eq:sketch-pac}
\end{align}
To turn \eqref{eq:sketch-pac} into an explicit bound, we now choose $\rho$ to be the Gibbs posterior
\[
\rho_n^\star(d\theta)
\;\propto\;
\exp\Bigl\{-\frac{(1+\omega L/2)\,\omega n}{2}\,R(\theta,\theta^\star)\Bigr\}\,\varphi(d\theta).
\]
For this choice, the risk term and the Kullback--Leibler divergence in \eqref{eq:sketch-pac} combine into a single marginal-type integral:
\begin{align}
\mb{E}_{\theta\sim\rho_n^\star}\Bigl\{\tfrac{(1+\omega L/2)\,\omega n}{2}\,R(\theta,\theta^\star)\Bigr\}
+
\mathrm{KL}(\rho_n^\star\|\varphi) =
-\log \int_\Theta \exp\Bigl\{-\tfrac{(1+\omega L/2)\,\omega n}{2}\,R(\theta,\theta^\star)\Bigr\}\,\varphi(d\theta).
\label{eq:sketch-log-partition}
\end{align}
The PAC-Bayes bound \eqref{eq:sketch-pac} thus reduces to controlling a single Laplace-type integral of the form \eqref{eq:sketch-log-partition}. By the RLCT-based expansion in Section~\ref{sec:SLT}, this integral satisfies
\[
-\log \int_\Theta \exp\Bigl\{-\tfrac{(1+\omega L/2)\,\omega n}{2}\,R(\theta,\theta^\star)\Bigr\}\,\varphi(d\theta)
=
\lambda \log n - (m-1)\log\log n + C_0(\varphi,L,\omega),
\]
where the scaling factor $\tfrac{(1+\omega L/2)\omega}{2}$ only affects the constant $C_0(\varphi,L,\omega)$ and not the coefficients of $\log n$ or $\log\log n$. Substituting this expression back into \eqref{eq:sketch-pac} yields the claimed bound \eqref{eq:explicit}.

\medskip
Now, it suffices to establish the two PAC-Bayes inequalities \eqref{eq:R bound by Rn} and \eqref{eq:Rn bound by R}.
We first introduce the Donsker--Varadhan variational formula (Lemma 2.2 in \citet{alquier2021user}): 
for any measurable $h$ with $e^h\in L_1(\varphi)$ and any $\rho\ll\varphi$,
\begin{align}
    \log \int e^{h(\theta)}\,\varphi(d\theta)
    \ge
    \int h(\theta)\,\rho(d\theta)
    - \mathrm{KL}(\rho\|\varphi).
    \label{DV-formula}
\end{align}

\paragraph{Control of the population risk by the empirical risk.} Next we establish an exponential bound from Assumption~\ref{assump:overall}. 
By inequality \eqref{eq:sketch-mgf-minus}, for any $\omega\in(0,\bar\omega)$, $\mb{E}_Z\Big[\exp\big\{-\omega\ell(\theta,\theta^\star;Z)\big\}\Big] \le \exp\big\{-(1-\tfrac{\omega L}{2})\omega\,R(\theta,\theta^\star)\big\}$.
Since $Z_1,\dots,Z_n$ are i.i.d., this implies that, for any fixed $\theta$,
\[
\mb{E}_{Z^n}\Big[\exp\big\{-\omega n R_n(\theta,\theta^\star)\big\}\Big]
=
\prod_{i=1}^n
\mb{E}_{Z_i}\Big[\exp\big\{-\omega\,\ell(\theta,\theta^\star;Z_i)\big\}\Big]
\le
\exp\big\{-(1-\tfrac{\omega L}{2})\omega n R(\theta,\theta^\star)\big\}.
\]
Equivalently,
\begin{align}
    \mb{E}_{Z^n}\Big[\exp\big\{\!-\omega n R_n(\theta,\theta^\star) + (1-\tfrac{\omega L}{2})\omega n R(\theta,\theta^\star)\big\}\Big] \le 1. \label{eq:pointwise-mgf}
\end{align}
Integrating both sides of \eqref{eq:pointwise-mgf} with respect to the prior $\varphi$ and applying Fubini's theorem, we obtain
\begin{align}
    \mb{E}_{Z^n} \int \exp\big\{-\omega n R_n(\theta,\theta^\star) +(1-\tfrac{\omega L}{2})\omega n R(\theta,\theta^\star)\big\} \,\varphi(d\theta) \;\le\; 1. \label{eq:mgf-integrated}
\end{align}
Now fix an arbitrary probability measure $\rho\ll\varphi$. For each realization of $Z^n$, define $h_{Z^n}(\theta) = -\omega n R_n(\theta,\theta^\star) + (1-\tfrac{\omega L}{2})\omega n R(\theta,\theta^\star)$. Then choosing $\rho = \Pi_n$ and applying the inequality \eqref{DV-formula} with this random function $h_{Z^n}$,
\[
\int e^{h_{Z^n}(\theta)}\,\varphi(d\theta) \;\ge\; \exp\Big\{ \int h_{Z^n}(\theta)\,\Pi_n(d\theta) - \mathrm{KL}(\Pi_n\|\varphi) \Big\}.
\]
Taking expectations with respect to $Z^n$ and combining with  \eqref{eq:mgf-integrated} yields
\[
\mb{E}_{Z^n} \exp\Big\{\int\!\big[-\omega n R_n(\theta,\theta^\star)  + (1-\tfrac{\omega L}{2})\omega n R(\theta,\theta^\star)\big]\Pi_n(d\theta) - \mathrm{KL}(\Pi_n\|\varphi) \Big\} \;\le\; 1.
\]
By Markov's inequality, for any $\delta\in(0,1)$, this implies that with probability
at least $1-\delta$,
\[
\int\!\big[-\omega n R_n(\theta,\theta^\star) + (1-\tfrac{\omega L}{2})\omega n R(\theta,\theta^\star)\big]\Pi_n(d\theta) - \mathrm{KL}(\Pi_n\|\varphi) \;\le\; \log(1/\delta).
\]
Rearranging terms gives
\[
(1-\tfrac{\omega L}{2})\omega n \int R(\theta,\theta^\star)\,\Pi_n(d\theta)\;\le\; \omega n \int R_n(\theta,\theta^\star)\,\Pi_n(d\theta) + \mathrm{KL}(\Pi_n\|\varphi) + \log(1/\delta).
\]
Through the interpretation of $\Pi_n$ given in \eqref{eq:interp Pi n} achieves the minimum of the right hand side. Therefore, it is exactly the claimed PAC-Bayes inequality after dividing both sides
by $(1-\tfrac{\omega L}{2})\omega n$.

\paragraph{Control of the empirical risk by the population risk.} Fix any probability measure $\rho\ll\varphi$ and $\omega\in(0,\bar\omega)$.
According to the upper exponential moment bound \eqref{eq:sketch-mgf-plus}, since $Z_1,\dots,Z_n$ are i.i.d., then for any fixed $\theta$,
\[
\mb{E}_{Z^n}\Big[\exp\big\{\omega n R_n(\theta,\theta^\star)\big\}\Big]
=
\prod_{i=1}^n
\mb{E}_{Z_i}\Big[\exp\big\{\omega\,\ell(\theta,\theta^\star;Z_i)\big\}\Big]
\le
\exp\big\{(1+\tfrac{\omega L}{2})\omega n R(\theta,\theta^\star)\big\}.
\]
Equivalently,
\begin{align}
    \mb{E}_{Z^n}\Big[\exp\big\{\omega n R_n(\theta,\theta^\star) -(1+\tfrac{\omega L}{2})\omega n R(\theta,\theta^\star)\big\}\Big] \le 1.
    \label{eq:pointwise-mgf-upper}
\end{align}
Integrating both sides with respect to the prior $\varphi$ and applying Fubini's theorem, we obtain
\begin{align}
    \mb{E}_{Z^n}
    \int 
    \exp\big\{\omega n R_n(\theta,\theta^\star)
               - (1+\tfrac{\omega L}{2})\omega n R(\theta,\theta^\star)\big\}
    \,\varphi(d\theta)
    \;\le\; 1.
    \label{eq:mgf-integrated-upper}
\end{align}

For each realization of $Z^n$, define the random function $h_{Z^n}(\theta) = \omega n R_n(\theta,\theta^\star) - (1+\tfrac{\omega L}{2})\omega n R(\theta,\theta^\star)$.
Now applying \eqref{DV-formula} with $h=h_{Z^n}$, we obtain for every $Z^n$,
\[
\int e^{h_{Z^n}(\theta)}\,\varphi(d\theta)
\;\ge\;
\exp\Big\{
\int h_{Z^n}(\theta)\,\rho(d\theta)
-
\mathrm{KL}(\rho\|\varphi)
\Big\}.
\]
Taking expectations with respect to $Z^n$ and combining with
\eqref{eq:mgf-integrated-upper} yields
\[
\mb{E}_{Z^n} \exp\Big\{ \omega n \int R_n(\theta,\theta^\star)\,\rho(d\theta) - (1+\tfrac{\omega L}{2})\omega n \int R(\theta,\theta^\star)\,\rho(d\theta) - \mathrm{KL}(\rho\|\varphi) \Big\}\;\le\; 1.
\]

By Markov's inequality, for any $\delta\in(0,1)$ this implies that, with probability
at least $1-\delta$,
\[
\omega n \int R_n(\theta,\theta^\star)\,\rho(d\theta)
-
(1+\tfrac{\omega L}{2})\omega n \int R(\theta,\theta^\star)\,\rho(d\theta)
-
\mathrm{KL}(\rho\|\varphi)
\;\le\;
\log(1/\delta).
\]
Finally, rearranging the terms gives
\[
\int R_n(\theta,\theta^\star)\,\rho(d\theta)
\;\le\;
(1+\tfrac{\omega L}{2}) \int R(\theta,\theta^\star)\,\rho(d\theta)
+
\frac{\mathrm{KL}(\rho\|\varphi)+\log(1/\delta)}{\omega n}.
\]

Combining the two PAC-Bayesian inequalities \eqref{eq:R bound by Rn} and \eqref{eq:Rn bound by R} with the RLCT-based expansion of the log-partition function completes the proof.
\end{proof}

\subsection{Auxiliary lemmas for Theorem~\ref{thm:completion}}

Before proving Theorem~\ref{thm:completion}, we first introduce some technical lemmas that help us transfer the population excess risk into a normal crossing form.

\begin{lemma}\label{lem:canonical_M*}
    Let $P_0, Q_0$ be two invertible $d_1\times d_1$ and $d_2\times d_2$ matrices, respectively such that 
    \[
    M^\star = P_0\begin{pmatrix}
        I_r & 0\\
        0 & 0
    \end{pmatrix}Q_0.
    \]
    Then by denoting $U' = P_0^{-1}U$, $V' = VQ_0^{-1}$ and $M' = P_0^{-1} M Q_0^{-1}$ and $M_r = \begin{pmatrix} I_r & 0 \\ 0 & 0 \end{pmatrix}$,
    \begin{align*}
        \sigma^2_{\min}(P_0) \sigma^2_{\min}(Q_0) \left\|M'- M_r\right\|_F^2\le\|M - M^\star\|_F^2\le \sigma^2_{\max}(P_0) \sigma^2_{\max}(Q_0)\left\|M' - M_r\right\|_F^2,
    \end{align*}
    where $\sigma_{\min}(\cdot)$ and $\sigma_{\max}(\cdot)$ represent the smallest and largest singular values of a matrix respectively. Moreover, the mapping $(U,V)\mapsto(U',V')$ is an invertible linear change of variables on $\mathbb R^{d_1 H + H d_2}$ whose Jacobian determinant is $|J| = |\det(P_0)|^{-H}|\det(Q_0)|^{-H}$.
\end{lemma}

\noindent \begin{proof}
By the definitions of $M^\star$ and $M'$, we have
\[
M^\star = P_0 M_r Q_0 \quad \text{and} \quad M = P_0 M' Q_0.
\]
Subtracting these two identities yields
\[
M - M^\star
= P_0 M' Q_0 - P_0 M_r Q_0
= P_0 (M' - M_r) Q_0.
\]

Using the sub-multiplicativity of the Frobenius norm with respect to the operator norm, we obtain
\begin{align*}
\|M - M^\star\|_F
&= \|P_0 (M' - M_r) Q_0\|_F \\
&\le \|P_0\|_{\mathrm{op}}\,\|(M' - M_r) Q_0\|_F 
\;\le\; \|P_0\|_{\mathrm{op}}\,\|M' - M_r\|_F\,\|Q_0\|_{\mathrm{op}}.
\end{align*}
Since for the Euclidean norm $\|A\|_{\mathrm{op}} = \sigma_{\max}(A)$, we square both sides and obtain the upper bound
\begin{equation}\label{eq:upper_canonical}
    \|M - M^\star\|_F^2
    \;\le\;
    \sigma_{\max}^2(P_0)\,\sigma_{\max}^2(Q_0)\,\|M' - M_r\|_F^2.
\end{equation}

Conversely, from $M = P_0 M' Q_0$ and $M^\star = P_0 M_r Q_0$ we have
\[
M' - M_r = P_0^{-1} (M - M^\star) Q_0^{-1}.
\]
Taking Frobenius norms and applying sub-multiplicativity again gives
\begin{align*}
\|M' - M_r\|_F
&= \|P_0^{-1} (M - M^\star) Q_0^{-1}\|_F \le \|P_0^{-1}\|_{\mathrm{op}}\,\|M - M^\star\|_F \,\|Q_0^{-1}\|_{\mathrm{op}}.
\end{align*}
Using $\|A^{-1}\|_{\mathrm{op}} = 1/\sigma_{\min}(A)$, this becomes
\[
\|M' - M_r\|_F
\;\le\;
\frac{1}{\sigma_{\min}(P_0)\,\sigma_{\min}(Q_0)}\,\|M - M^\star\|_F,
\]
so that
\[
\sigma_{\min}(P_0)\,\sigma_{\min}(Q_0)\,\|M' - M_r\|_F
\;\le\;
\|M - M^\star\|_F.
\]
Squaring both sides yields the lower bound
\begin{equation}\label{eq:lower_canonical}
    \sigma_{\min}^2(P_0)\,\sigma_{\min}^2(Q_0)\,\|M' - M_r\|_F^2
    \;\le\;
    \|M - M^\star\|_F^2.
\end{equation}
Combining \eqref{eq:upper_canonical} and \eqref{eq:lower_canonical} gives the desired two-sided inequality.

It remains to compute the Jacobian of the mapping $(U,V)\mapsto(U',V')$.
We have, in vectorized form,
\[
\mathrm{vec}(U') = (I_H \otimes P_0^{-1})\,\mathrm{vec}(U),
\qquad
\mathrm{vec}(V') = ((Q_0^{-1})^{\top} \otimes I_H)\,\mathrm{vec}(V),
\]
where $\otimes$ denotes the Kronecker product. Thus the linear map on
$\mb R^{d_1 H + H d_2}$ corresponding to $(U,V)\mapsto(U',V')$ is block-diagonal
with blocks $I_H\otimes P_0^{-1}$ and $(Q_0^{-1})^{\top}\otimes I_H$, so its Jacobian determinant is
\begin{align*}
|J|
&= \bigl|\det(I_H\otimes P_0^{-1})\bigr|\,
   \bigl|\det((Q_0^{-1})^{\top}\otimes I_H)\bigr| \\
&= |\det(P_0^{-1})|^H\,|\det((Q_0^{-1})^{\top})|^H
= |\det(P_0)|^{-H}\,|\det(Q_0)|^{-H},
\end{align*}
which is a positive constant depending only on $P_0$ and $Q_0$. This completes the proof.
\end{proof}

\begin{lemma}\label{lem:finite_operator}
Let $A_1\in\mathbb{R}^{H_3\times H_2}$, $A_2\in\mathbb{R}^{H_3\times H_1}$ and $A\in\mathbb{R}^{H_2\times H_1}$ be three given matrices.
Define $D := \max\bigl\{\,2\|A\|_F^2 + 1,\ 2\,\bigr\}$. Then
\[
\frac{1}{D}\bigl(\|A_1\|_F^2 + \|A_2\|_F^2\bigr)
\;\le\;
\|A_1\|_F^2 + \|A_2 + A_1A\|_F^2
\;\le\;
D\bigl(\|A_1\|_F^2 + \|A_2\|_F^2\bigr).
\]
\end{lemma}

\noindent \begin{proof}
For the upper bound, we use the elementary inequality
\[
\|X+Y\|_F^2 \le 2\|X\|_F^2 + 2\|Y\|_F^2
\]
and the sub-multiplicativity of the Frobenius norm,
\[
\|A_1A\|_F \le \|A_1\|_F\,\|A\|_F.
\]
Thus
\begin{align*}
\|A_1\|_F^2 + \|A_2 + A_1A\|_F^2
&\le \|A_1\|_F^2 + 2\|A_2\|_F^2 + 2\|A_1A\|_F^2 \\
&\le (2\|A\|_F^2 + 1)\,\|A_1\|_F^2 + 2\|A_2\|_F^2.
\end{align*}
By the definition of $D$, we have $D \ge 2\|A\|_F^2 + 1$ and $D \ge 2$, hence
\[
(2\|A\|_F^2 + 1)\,\|A_1\|_F^2 + 2\|A_2\|_F^2
\le D\|A_1\|_F^2 + D\|A_2\|_F^2
= D\bigl(\|A_1\|_F^2 + \|A_2\|_F^2\bigr),
\]
which proves the desired upper bound.

For the lower bound, write $A_2 = (A_2 + A_1A) - A_1A$. Then
\begin{align*}
\|A_2\|_F^2
&= \|(A_2 + A_1A) - A_1A\|_F^2 \\
&\le 2\|A_2 + A_1A\|_F^2 + 2\|A_1A\|_F^2 \\
&\le 2\|A_2 + A_1A\|_F^2 + 2\|A\|_F^2\,\|A_1\|_F^2.
\end{align*}
Adding $\|A_1\|_F^2$ to both sides yields
\[
\|A_1\|_F^2 + \|A_2\|_F^2
\le 2\|A_2 + A_1A\|_F^2 + (2\|A\|_F^2 + 1)\,\|A_1\|_F^2.
\]
Again using $D \ge 2$ and $D \ge 2\|A\|_F^2 + 1$, we have
\[
\|A_1\|_F^2 + \|A_2\|_F^2
\le D\|A_2 + A_1A\|_F^2 + D\|A_1\|_F^2
= D\bigl(\|A_1\|_F^2 + \|A_2 + A_1A\|_F^2\bigr).
\]
Dividing both sides by $D$ gives
\[
\frac{1}{D}\bigl(\|A_1\|_F^2 + \|A_2\|_F^2\bigr)
\le
\|A_1\|_F^2 + \|A_2 + A_1A\|_F^2,
\]
which is the claimed lower bound.
\end{proof}

\begin{lemma}[Modified Theorem 4.7 in \citet{watanabe2009algebraic}]\label{lemma:thm 4.7}
    Let $m,r$ be natural numbers. Let $k,h\in\mb Z_{\ge0}^m$ and $k',h'\in\mb Z_{\ge0}^r$ be multi-indices satisfying
\[
\frac{h_1+1}{2k_1} = \frac{h_2+1}{2k_2} = \cdots = \frac{h_m+1}{2k_m} =:\lambda \qquad \text{and} \qquad \frac{h'_j+1}{2k'_j} > \lambda, \qquad j=1,\dots,r.
\]
Fix a constant $\beta>0$, and define for $n>1$
\[
Z(n) =
\int_{[0,1]^m}dx
\int_{[0,1]^r}dy\;
\exp\bigl\{-n\beta\, x^{2k} y^{2k'}\bigr\}\,
x^h y^{h'},
\]
where
\(
x^h=\prod_{i=1}^m x_i^{h_i},
\ y^{h'}=\prod_{j=1}^r y_j^{h'_j},
\ x^{2k}=\prod_{i=1}^m x_i^{2k_i},
\ y^{2k'}=\prod_{j=1}^r y_j^{2k'_j}.
\)
Then for all $n>1$,
\[
Z(n) \;\ge\;
\frac{(\log n)^{m-1}}{n^\lambda} \cdot \prod_{j=1}^m \frac{1}{h'_j+1}\cdot \frac{1}{2^m(m-1)! k_1k_2\cdots k_m} \cdot \frac{1}{\lambda e^\beta}.
\]
\end{lemma}

\noindent \begin{proof}
    The evaluation of singular integrals of the form $\int_{[0,1]^m} \exp\bigl(-a x^{2k}\bigr)\,x^{h}\,dx$
    is conveniently handled via the \emph{state density function} introduced in \citet{watanabe2009algebraic}.
    For fixed exponents $k_1,\dots,k_m\in\mb Z_{\ge0}$ and $h_1,\dots,h_m\in\mb Z_{\ge0}$, and a scaling parameter $a>0$, define
    \[
        v(t)
        := \int_{[0,1]^r} \delta\bigl(t - a x^{2k}\bigr)\,x^{h}\,dx,
    \]
    where $\delta(\cdot)$ denotes the Dirac delta function: for any continuous test function $\varphi$ and fixed $x$,
    \[
        \int_0^\infty \varphi(t)\,\delta\bigl(t-a x^{2k}\bigr)\,dt
        = \varphi\bigl(a x^{2k}\bigr).
    \]
    Informally, the factor $\delta(t-a x^{2k})$ concentrates all mass at the level set where $a x^{2k}=t$, so that $v(t)\,dt$ measures the volume of those parameters $x\in[0,1]^m$ for which $a x^{2k}$ falls into the interval $[t,t+dt]$.
    The advantage of introducing $v(t)$ is that it reduces the original $m$–dimensional integral to a one–dimensional integral in $t$, and Theorem~4.6 in \citet{watanabe2009algebraic} gives an explicit expression for $v(t)$, making the dependence on $a$ (and hence on $n$ in our setting) completely transparent.

    More precisely, assume that there exists a rational number $\lambda>0$ such that
    \[
        \frac{h_1+1}{2k_1}
        = \frac{h_2+1}{2k_2}
        = \cdots
        = \frac{h_r+1}{2k_r}
        = \lambda.
    \]
    Then, by Theorem~4.6 of \citet{watanabe2009algebraic}, there exists a constant
    \[
        \gamma
        = \frac{1}{2^m(m-1)!k_1\cdots k_m} \;>\;0
    \]
    such that
    \[
        v(t)
        =
        \begin{cases}
            \gamma\,\dfrac{t^{\lambda-1}}{a^{\lambda}}
            \Bigl(\log \dfrac{a}{t}\Bigr)^{m-1}, & 0<t<a,\\[1ex]
            0, & \text{otherwise}.
        \end{cases}
    \]
    This formula shows that, once the exponents $(h,k)$ are fixed, the state density has the small–$t$ behavior
    \(t^{\lambda-1}(\log(a/t))^{m-1}/a^\lambda\), and thus any integral involving $x^{2k}$ can be rewritten as a one–dimensional integral whose leading $a$–dependence is explicit.

    We now apply this representation to bound $Z(n)$.
    Since $y^{2k'}\le 1$ on $[0,1]^m$, we have
    \begin{align*}
        Z(n)
        &\ge \int_{[0,1]^m} dx \int_{[0,1]^r} dy\;
        \exp\bigl(-n\beta x^{2k}\bigr)\,x^{h}y^{h'}\\
        &= \prod_{j=1}^m \frac{1}{h'_j+1} \int_{[0,1]^m} dx\;
        \exp\bigl(-n\beta x^{2k}\bigr)\,x^{h}\\
        &= \prod_{j=1}^m \frac{1}{h'_j+1} 
           \int_0^\infty dt \int_{[0,1]^m} dx\;
           \delta\bigl(t - n x^{2k}\bigr)\,e^{-\beta t} x^{h}.
    \end{align*}
    Define
    \[
        v_n(t)
        := \int_{[0,1]^m} \delta\bigl(t - n x^{2k}\bigr)\,x^{h}\,dx.
    \]
    With $a=n$, the above applies to $v_n(t)$ and yields
    \[
        v_n(t)
        =
        \begin{cases}
            \gamma_1\,\dfrac{t^{\lambda-1}}{n^{\lambda}}
            \Bigl(\log \dfrac{n}{t}\Bigr)^{m-1}, & 0<t<n,\\[1ex]
            0, & \text{otherwise},
        \end{cases}
    \]
    where
    \[
        \gamma_1
        = \dfrac{1}{2^m(m-1)!k_1\cdots k_m}>0.
    \]
    Substituting this expression into the previous display, we obtain
    \begin{align*}
        Z(n)
        & \ge\; \prod_{j=1}^m \frac{1}{h'_j+1}\cdot \gamma_1 \int_0^n dt\;
    \frac{t^{\lambda-1}}{n^{\lambda}}
    \Bigl(\log\frac{n}{t}\Bigr)^{m-1} e^{-\beta t}\\
    & \ge\; \prod_{j=1}^m \frac{1}{h'_j+1}\cdot \gamma_1 \int_0^1 dt\;
    \frac{t^{\lambda-1}}{n^{\lambda}}
    (\log n )^{m-1} e^{-\beta t}\\
    & = \frac{(\log n)^{m-1}}{n^\lambda} \cdot \prod_{j=1}^m \frac{1}{h'_j+1}\cdot \frac{\gamma_1}{\lambda e^\beta},
    \end{align*}
    which completes the proof.
\end{proof}

\subsection{Proof of Theorem~\ref{thm:completion}}\label{app:matrix completion}
The structure of this proof is analogous to the argument in \citet{aoyagi2005stochastic} that derives the RLCT for reduced rank regression. Here we supply the omitted intermediate steps, and explicitly track the ranges of all parameters, the Jacobians of the transformations, and the Frobenius-norm–related factor that contributes to the constant term in Theorem~\ref{thm:completion}.

As we discussed in Section~\ref{subsec:matrix completion}, the population excess risk is expressed as
\[
R(M, M^\star) := \mb{E}[\ell(M, M^\star;Z)] = \frac{1}{d_1 d_2} \| M - M^\star \|_F^2 = \frac{1}{d_1 d_2} \| UV - M^\star \|_F^2.
\]
Assumption~\ref{assump:overall} has been verified for centered sub-Gaussian noise $\varepsilon$ in the main context. Then by Theorem~\ref{thm:main},
to derive the explicit PAC-Bayes bound for $\| M - M^\star \|_F^2$, we need to derive the normal crossing form of $\| UV - M^\star \|_F^2$ and track the Jacobian of the resolution map. 
According to Lemma~\ref{lem:canonical_M*}, $\|UV-M^\star\|_F^2$ is two-sidedly bounded by its canonical form  $\|U'V' - M_r\|$, with the Jacobian determinant being $|J| = |\det(P_0)|^{-H}|\det(Q_0)|^{-H}>0$. Therefore, it suffices to focus on the transformation and resolution of singularities of the latter. 

The proof proceeds in three steps. 
First, we apply a sequence of invertible transformations to rewrite the squared Frobenius norm into a form that is amenable to blow-up analysis. 
Second, we perform a blow-up to derive explicit expressions for the RLCT $\lambda$ and its multiplicity $m$. 
Finally, we invoke Theorem~4.7 of \citet{watanabe2009algebraic} to obtain the constant term $C_1$.

\paragraph{Step 1.}
To determine the RLCT, it suffices to restrict attention to a local neighborhood of
$M_r$. Indeed, for any $\eta>0$, the integral defining the zeta function is holomorphic
on the complement region $\{M' : \|M'-M_r\|_\infty \ge \eta\}$ and hence contributes no
poles. Therefore, only a local analysis around $M_r$ is relevant.

In what follows, we focus on the non-trivial case
$0<r<H\le\min\{d_1,d_2\}$.
After applying the rank-normalizing transformations
$U' = P_0^{-1}U$ and $V' = VQ_0^{-1}$,
we block-partition
\[
U' =
\begin{pmatrix}
U_1 & U_3\\
U_2 & U_4
\end{pmatrix},\qquad
V' =
\begin{pmatrix}
V_1 & V_3\\
V_2 & V_4
\end{pmatrix},\qquad
M' =
\begin{pmatrix}
M_1 & M_3\\
M_2 & M_4
\end{pmatrix},
\]
where $U_1,V_1,M_1$ are $r\times r$ blocks.

By assumption, the transformed parameter space is chosen sufficiently large so that
all configurations considered below lie within its support.
We further restrict attention to the region \(\|V_1-I_r\|_\infty \le \frac{1}{2r}\),
under which $V_1$ is invertible and satisfies \((1/2)^r \le |\det(V_1)| \le (3/2)^r\).
On this region, define the change of variables
\[
\widehat U_1 := U_1V_1 + U_3V_2 - M_r,\quad
\widehat U_2 := U_2V_1 + U_4V_2,\quad
\widehat V_4 := V_4 - V_2V_1^{-1}V_3,\quad
\widehat V_3 := V_1^{-1}V_3 + U_3\widehat V_4.
\]
This transformation is smooth and invertible. Writing the change of variables in the
direction
\(
(U_1,U_2,V_3,V_4)\;\mapsto\;(\widehat U_1,\widehat U_2,\widehat V_3,\widehat V_4)
\)
with $(U_3,U_4,V_1,V_2)$ held fixed, the induced Jacobian factor satisfies
\begin{align}
    dU_1\,dU_2\,dV_3\,dV_4
& =
|\det(V_1)|^{\,d_2-d_1-r}\,
d\widehat U_1\,d\widehat U_2\,d\widehat V_3\,d\widehat V_4 \nonumber\\
& \ge 2^{-r|d_2-d_1-r|}\,d\widehat U_1\,d\widehat U_2\,d\widehat V_3\,d\widehat V_4.\label{eq:Jacobian1}
\end{align}
Substituting the reparametrization into $U'V'-M_r$ yields
\[
U'V'-M_r
=
\begin{pmatrix}
\widehat U_1 &
\widehat U_1(\widehat V_3-U_3\widehat V_4)+\widehat V_3\\
\widehat U_2 &
\widehat U_2(\widehat V_3-U_3\widehat V_4)+U_4\widehat V_4
\end{pmatrix}.
\]

After the change of variables, the parameters are
\[
\theta' = (\widehat U_1,\widehat U_2,U_3,U_4,V_1,V_2,\widehat V_3,\widehat V_4),
\]
among which only
$\widehat U_1,\widehat U_2,U_3,U_4,\widehat V_3,\widehat V_4$
appear in the expression of $U'V'-M_r$.
We now restrict attention to a subset $\mathcal B$ of the transformed parameter space
defined by the bounds
\[
\|\widehat U_1\|_\infty,\ \|\widehat U_2\|_\infty,\ \|U_4\|_\infty, \ \|\widehat V_4\|_\infty
\le 1,
\qquad \ \|V_2\|_\infty\le \frac{1}{2},
\]

\[
\|\widehat V_3\|_\infty \le \frac{1}{2r(d_2-r)},
\qquad
\|U_3\|_\infty \le \frac{1}{2r(H-r)(d_2-r)}.
\]
Under these bounds,
\[
\|\widehat V_3-U_3\widehat V_4\|_\infty \le \frac{1}{r(d_2-r)}\implies
\|\widehat V_3-U_3\widehat V_4\|_F^2 \le 1.
\]
Therefore, by Lemma~\ref{lem:finite_operator},
\begin{align}
    \|U'V'-M_r\|_F^2 \le 3\Big(\|\widehat U_1\|_F^2 + \|\widehat U_2\|_F^2 + \|\widehat V_3\|_F^2  + \|U_4\widehat V_4\|_F^2 \Big) =: 3\cdot\Psi.\label{eq:psi}
\end{align}

Using the lower bound $\varphi_0$ of the prior density, the Jacobian estimate above,
and the definition of the excess risk, we obtain
\begin{align}
& \int_\Theta
\exp\Big\{-\tfrac{(1+\omega L/2)\,\omega n}{2}R(M,M^\star)\Big\}\,\varphi(d\theta)\nonumber \\
& \qquad\qquad\;\ge\;
\varphi_0\,
2^{-r|d_2-d_1-r|}\cdot
\int_{\mathcal B}
\exp\Big\{
-\tfrac{(1+\omega L/2)\,\omega n}{2d_1d_2}\cdot 3\,\Psi
\Big\}
\,d\theta'.
\label{eq:lower integral 1}
\end{align}
Finally, we rescale each entry in $V_1,V_2,\widehat V_3$ and $U_3$ to the interval $[-1,1]$, introducing an additional explicit scaling factor 
\begin{align}
    r^{-r^2}\cdot(2r(d_2-r))^{-r(d_2-r)}\cdot (2r(H-r)(d_2-r))^{-r(H-r)},\label{eq:rescale1}
\end{align}
corresponding to the side lengths of $\mathcal B$. As a result, the integral is reduced to this explicit prefactor multiplied by an integral over $[-1,1]^N$, where $N = (d_1+d_2)H$ is the total number of entries in $U'$ and $V'$,
For notational convenience, we continue to denote the rescaled variables by the same symbols and the considered rescaled local coordinate as $\m S$.

\paragraph{Step 2.}
We now analyze the squared Frobenius norm
\[
\Psi = \|\widehat U_1\|_F^2 + \|\widehat U_2\|_F^2
+ \frac{1}{4r^2(d_2-r)^2}\|\widehat V_3\|_F^2
+ \|U_4\widehat V_4\|_F^2,
\]
which is in a form amenable to resolution of singularities. Recall from Step~1 that, after restricting to $\mathcal B$ and rescaling, all entries of the variables appearing in $\Psi$ take values in $[-1,1]$. 
For notational convenience, we keep the same symbols for these rescaled variables.

We begin by constructing a blow-up of $\Psi$ along the submanifold \(\{\widehat U_1=0, \widehat U_2=0, \widehat V_3=0, \widehat V_4=0\}\).
Here we focus on describing the explicit blow-up procedure and the resulting local
coordinate representations; standard references on the general framework can be found
in \citet{watanabe2009algebraic}.

\paragraph{First blow-up step.} To describe the blow-up explicitly, we work on a local chart in which one entry of
$\widehat U_1$ serves as the radial variable.
Without loss of generality, we choose the $(1,1)$-entry and set
\[
(\widehat U_1)_{11}=b_1,\qquad b_1\in[-1,1].
\]
We then parametrize the remaining variables by
\begin{align}
    (\widehat U_1)_{ij}=b_1\,\widetilde U_{1,ij}\quad((i,j)\neq(1,1)),\qquad
\widehat U_2=b_1\,\widetilde U_2,\qquad
\widehat V_3=b_1\,\widetilde V_3,\qquad
\widehat V_4=b_1\,\widetilde V_4.\label{eq:blow1}
\end{align}

Since the original variables are restricted to $[-1,1]$, the rescaled variables
$\widetilde U_{1,ij},\widetilde U_2,\widetilde V_3,\widetilde V_4$
can also be taken to range in $[-1,1]$, without exceeding the original parameter
bounds. We denote this new local chart as $\m S'$.
Substituting this parametrization into $\Psi$ yields
\begin{align*}
\Psi \,& =\, \|\widehat U_1\|_F^2 + \|\widehat U_2\|_F^2 + \frac{1}{4r^2(d_2-r)^2} \|\widehat V_3\|_F^2 + \|U_4\widehat V_4\|_F^2\\
\,& \le\, b_1^2\Big(1 + \sum_{(i,j)\ne(1,1)} (\widetilde U_{1,ij})^2  + \|\widetilde U_2\|_F^2 + \|\widetilde V_3\|_F^2 + \|U_4\widetilde V_4\|_F^2
\Big),
\end{align*}
where the expression in parentheses is bounded below by $1$ and hence is uniformly bounded away from zero. We note that, in the argument below, whenever a cube of side length less than $1$ is rescaled to the unit cube, a prefactor smaller than $1$ appears in front of the rescaled term like the $1/[4r^2(d_2-r)^2]$ here. Hence, if we simply discard this prefactor, the resulting expression can only increase, and therefore still provides a valid upper bound for $\Psi$. Consequently, $\Psi$ is already in normal crossing form in this chart.
Moreover, on the region where all entries of the rescaled variables lie in $[-1,1]$,
the sum in parentheses admits the explicit upper bound
\begin{align}
    1 + (r^2-1) + (d_1-r)r + r(d_2-r) + (d_1-r)(d_2-r)(H-r)^2\le d_1d_2(H-r)^2.\label{eq:psi multiple1}
\end{align}

Indeed, the first three Frobenius norms are bounded by the numbers of their entries
since each entry is in $[-1,1]$, while $U_4\widetilde V_4$ is a
$(d_1-r)\times(d_2-r)$ matrix whose $(i,j)$-th entry is a sum of $(H-r)$ products,
hence satisfies $|(U_4\widetilde V_4)_{ij}|\le H-r$ and thus
\(
\|U_4\widetilde V_4\|_F^2\le (d_1-r)(d_2-r)(H-r)^2.
\)
Consequently, $\Psi$ admits a normal crossing form in this chart.
The Jacobian of the change of variables in \eqref{eq:blow1} is $|b_1|^{\,h(0)-1}$, where for any $0\le t\le H-r$, we define
\begin{align}
    h(t)=r(d_1+d_2-r)+t(d_1-r)+(H-r-t)(d_2-r-t).\label{eq:h(t)}
\end{align}
Here $h(t)$ is a function introduced for convenience to express the RLCT, note that $h(0)$ is the total dimension in $(\widehat U_1, \widehat U_2, \widehat V_3, \widehat V_4)$. It essentially represents the number of free parameters at each step. Since $h(t)$ is a quadratic function of $t$ with positive leading coefficient, it is convex in $t$, and therefore it achieves maximum at $t=0$ or $H-r$,
\begin{align}
    h(t)\;\le\;
r(d_1+d_2-r) + (H-r)\cdot \max\{d_1-r, d_2-r\}\le H(d_1+d_2-r).\label{eq:h(t) upper}
\end{align}

In this chart, it yields a candidate pole at
\begin{align}
    \lambda_{b_1} = \frac{(h(0)-1)+1}{2} = \frac{h(0)}{2},\qquad m_{b_1} = 1.\label{eq:pole1}
\end{align}
By symmetry, choosing any single entry of $\widehat U_1$, $\widehat U_2$, or
$\widehat V_3$ as the radial variable leads to the same contribution and hence the
same candidate pole.

We now consider blow-ups along directions corresponding to entries of $\widehat V_4$.
By symmetry, it suffices to analyze a local chart in which the $(1,1)$-entry of
$\widehat V_4$ serves as the radial variable.
We write
\[
U_4 = \bigl(\bm u_1,\; U_4^{(2)}\bigr)\in\mb R^{(d_1-r)\times(H-r)},
\qquad
\widehat V_4 =
\begin{pmatrix}
v_{11} & \bm v_{12}\\
\bm v_{21} & V_4^{(2)}
\end{pmatrix}\in\mb R^{(H-r)\times(d_2-r)},
\]
where $\bm u_1\in\mb R^{d_1-r}$ is the first column of $U_4$,
$v_{11}$ is the $(1,1)$-entry of $\widehat V_4$,
$\bm v_{12}\in\mb R^{1\times(d_2-r-1)}$ is the remainder of the first row,
$\bm v_{21}\in\mb R^{(H-r-1)\times 1}$ is the remainder of the first column,
and $V_4^{(2)}\in\mb R^{(H-r-1)\times(d_2-r-1)}$ is the lower-right block.

On this chart we take
\[
v_{11} = a_1,\quad a_1\in[-1,1],
\]
and perform the blow-up by scaling all other entries in 
$\{\widehat U_1,\widehat U_2,\widehat V_3,\widehat V_4\}$:
\begin{align}
    \widehat U_1 = a_1\,U^{(2)}_1,\qquad
\widehat U_2 = a_1\,U^{(2)}_2,\qquad
\widehat V_3 = a_1\,V^{(2)}_3,\nonumber\\
{\bm v}_{12} = a_1\,\widehat {\bm v}_{12},\qquad
{\bm v}_{21} = a_1\,\widehat {\bm v}_{21},\qquad
V_4^{(2)} = a_1\,\widehat V_4^{(2)}.\label{eq:blow up V4 1}
\end{align}
The Jacobian determinant of this step is $a_1^{h(0)-1}$. To make the product term explicit, note that
\begin{align*}
    U_4\widehat V_4
& =
\bigl(\bm u_1,\; U_4^{(2)}\bigr)
\begin{pmatrix}
a_1 & a_1\widehat {\bm v}_{12}\\
a_1\widehat {\bm v}_{21} & a_1\widehat V_4^{(2)}
\end{pmatrix}\\
& =
a_1\Bigl(
\bm u_1 + U_4^{(2)}\widehat {\bm v}_{21},\;\;
\bm u_1\widehat {\bm v}_{12} + U_4^{(2)}\widehat V_4^{(2)}
\Bigr).
\end{align*}

Let $\widehat{\bm u}_1 := \bm u_1 + U_4^{(2)}\widehat {\bm v}_{21}$ and $\widetilde V_4^{(2)} := \widehat V_4^{(2)} - \widehat {\bm v}_{21}\widehat {\bm v}_{12}$.
The Jacobian of this change of variables is equal to one, and in the new coordinates
the product reduces to
\[
U_4\widehat V_4
=
a_1\Bigl(\widehat{\bm u}_1,\; \widehat{\bm u}_1\widehat{\bm v}_{12} + U_4^{(2)}\widetilde V_4^{(2)}\Bigr).
\]

We now address the domain constraints and apply Lemma~\ref{lem:finite_operator} to
control the mixed terms.
Assume
\[
\|U_4^{(2)}\|_\infty\le \frac{1}{H-r},
\qquad
\|\widehat{\bm v}_{12}\|_\infty \le \frac{1}{2},
\]
so that
\[
\|U_4^{(2)}\widehat{\bm v}_{21}\|_\infty\le\frac{1}{2},
\qquad
\|\widehat{\bm v}_{21}\widehat{\bm v}_{12}\|_\infty\le \frac{1}{2}.
\]
Since all other coordinates remain in $[-1,1]$, these bounds imply that the current chart contains a hyperrectangle on which
\begin{align}
    \|U_4^{(2)}\|_\infty\le \frac{1}{H-r},
\qquad
\|\widehat{\bm v}_{12}\|_\infty \le \frac{1}{2},\qquad \|\widehat{\bm u}_1\|_\infty \le \frac{1}{2},
\qquad
\|\widetilde V_4^{(2)}\|_\infty \le \frac{1}{2},\label{eq:domain 2}
\end{align}
while the remaining variables range in $[-1,1]$. Moreover, since $\|\widehat{\bm u}_1\widehat{\bm v}_{12}\|_\infty\le 1/2$, its Frobenius norm
satisfies
\[
\|\widehat{\bm u}_1\widehat{\bm v}_{12}\|_F^2
\;\le\;
\frac{(d_1-r)(d_2-r-1)}{4}.
\]

Therefore, by Lemma~\ref{lem:finite_operator},
\begin{align*}
    \|U_4\widehat V_4\|_F^2
& \le\;
a_1^2\Bigl(\tfrac{(d_1-r)(d_2-r-1)}{2}+1\Bigr)
\bigl(\|\widehat{\bm u}_1\|_F^2 + \|U_4^{(2)}\widetilde V_4^{(2)}\|_F^2\bigr)\\ 
& \le\; a_1^2\,d_1d_2 \bigl(\|\widehat{\bm u}_1\|_F^2 + \|U_4^{(2)}\widetilde V_4^{(2)}\|_F^2\bigr).
\end{align*}
Therefore, we obtain, on this local chart,
\begin{align}
    \Psi \le
    a_1^2 d_1d_2
    \bigl(\|U^{(2)}_1\|_F^2 + \|U^{(2)}_2\|_F^2 + \|V^{(2)}_3\|_F^2
          + \|\widehat{\bm u}_1\|_F^2 + \|U_4^{(2)}\widetilde V_4^{(2)}\|_F^2\bigr).
    \label{eq:psi2}
\end{align}
Next, we rescale the active coordinates restricted as in \eqref{eq:domain 2} so that each ranges over $[-1,1]$ and denote the new local chart as $\m S_1$. This
change of variables contributes the Jacobian factor
\begin{align}
    (H-r)^{-(d_1-r)(H-r-1)}\cdot 2^{-d_1+r - (H-r)(d_2-r-1)}.
    \label{eq:rescale2}
\end{align}
On the right-hand side of \eqref{eq:psi2}, the squared norm $\|\widehat{\bm u}_1\|_F^2$, is now a regular quadratic form as well. Since the expression in \eqref{eq:psi2} is not yet in normal crossing form, we proceed with a second blow-up.

\paragraph{Second blow-up step.} At this stage there are two possible choices for the next blow-up along $\{U^{(2)}_1 =0, U^{(2)}_2 = 0,V^{(2)}_3 = 0, \widehat{\bm u}_1 = 0, \widetilde V_4^{(2)}=0\}$.
The first option is to perform a blow-up in one of the regular directions similarly as we did in \eqref{eq:blow1}. To describe this blow-up explicitly, without
loss of generality, we choose the $(1,1)$-entry of $U^{(2)}_1$ and set
\[
(U^{(2)}_1)_{11}=b_2,\qquad b_2\in[-1,1].
\]
We then parametrize the remaining variables by
\begin{align}
    (U^{(2)}_1)_{ij} =b_2\,\widetilde U^{(2)}_{1,ij},\;((i,j)& \neq(1,1)), \qquad
    U^{(2)}_2 =b_2\,\widetilde U^{(2)}_2,\qquad
    V^{(2)}_3=b_2\,\widetilde V^{(2)}_3, \nonumber\\ 
    \widehat{\bm u}_1& =b_2\,\widetilde{\bm u}_1,\qquad
    \widetilde V_4^{(2)}=b_2\,\bar V_4^{(2)}.
    \label{eq:blow2}
\end{align}
The Jacobian of this change of variables therefore contributes a factor $|b_2|^{\,h(1)-1}$. As defined in \eqref{eq:h(t)}, $h(1)$ denote the total number of coordinates in
$(U^{(2)}_1,U^{(2)}_2,V^{(2)}_3,\widehat{\bm u}_1,\widetilde V_4^{(2)})$. Since all these variables are already restricted to $[-1,1]$, the rescaled variables $\widetilde U^{(2)}_1,\widetilde U^{(2)}_2,\widetilde V^{(2)}_3,\widetilde{\bm u}_1,\bar V_4^{(2)}$
can also be taken to range in $[-1,1]$, without exceeding the original parameter
bounds.

Denote the new local chart as $\m S_1'$. Substituting \eqref{eq:blow2} into \eqref{eq:psi2} yields
\begin{align*}
\Psi \le a_1^2\, b_2^2\, d_1d_2\,
\Bigl(1 + \sum_{(i,j)\ne(1,1)} (\widetilde U^{(2)}_{1,ij})^2 + \|\widetilde U^{(2)}_2\|_F^2 + \|\widetilde V^{(2)}_3\|_F^2 + \|\widetilde{\bm u}_1\|_F^2 + \|U_4^{(2)}\bar V_4^{(2)}\|_F^2
\Bigr).
\end{align*}
The expression in parentheses is bounded below by $1$ and hence is uniformly bounded
away from zero. Thus, in the current local coordinate, $\Psi$ is already in normal
crossing form on this chart. The blow-ups along $a_1$ and $b_2$ therefore give rise to
two candidate poles
\[
\lambda_{a_1} = \frac{h(0)}{2},
\qquad
\lambda_{b_2} = \frac{h(1)}{2},
\]
with corresponding multiplicities
\[
m_{a_1} =
\begin{cases}
1, & \text{if } h(0) < h(1),\\
2, & \text{if } h(0) = h(1),
\end{cases}
\qquad
m_{b_2} =
\begin{cases}
1, & \text{if } h(1) < h(0),\\
2, & \text{if } h(1) = h(0),
\end{cases}
\]
depending on whether the exponents $h(0)$ and $h(1)$ coincide. In particular, the
smaller of $\lambda_{a_1}$ and $\lambda_{b_2}$ determines the leading pole contributed
by this chart, while equality $h(0)=h(1)$ results in an increased multiplicity.

Moreover, on the region where all rescaled variables take values in $[-1,1]$, the sum
in parentheses admits the explicit upper bound
\[
1 + (r^2-1) + (d_1-r)r + r(d_2-r) + (d_1-r)
+ (d_1-r)(d_2-r-1)(H-r-1)^2,
\]
where the first five terms come from the Frobenius norms of
$\widetilde U^{(2)}_1,\widetilde U^{(2)}_2,\widetilde V^{(2)}_3,\widetilde{\bm u}_1$, and the last term
bounds $\|U_4^{(2)}\widehat V_4^{(2)}\|_F^2$ using that each entry of
$U_4^{(2)}\widehat V_4^{(2)}$ is a sum of $(H-r-1)$ products of terms in $[-1,1]$.
As in \eqref{eq:psi multiple1}, this factor can be further bounded by
\[
r(d_1+d_2-r) + (d_1-r)
+ (d_1-r)(d_2-r-1)(H-r-1)^2
\;\le\;
d_1d_2(H-r)^2.
\]
Consequently, on this chart we obtain the bound
\begin{align}
    \Psi \;\le\; a_1^2 b_2^2\,(d_1d_2)^2(H-r)^2.\label{eq:psi3}
\end{align}

Now we return to equation \eqref{eq:psi2}. The second option is to carry out a further blow-up in the block $\widetilde V_4^{(2)}$.
By symmetry, it suffices to work on a local chart in which the upper-left entry of
$\widetilde V_4^{(2)}$ serves as the new radial variable. Since $\widetilde V_4^{(2)}$ is
obtained from $\widehat V_4$ by removing its first row and first column, for notational
convenience we continue to label its entries using their original positions in
$\widehat V_4$ in what follows.
 We write
\[
U_4^{(2)} = \bigl(\bm u_2,\; U_4^{(3)}\bigr)\in\mb R^{(d_1-r)\times(H-r-1)},
\qquad
\widetilde V_4^{(2)} =
\begin{pmatrix}
v_{22} & \bm v_{23}\\
\bm v_{32} & V_4^{(3)}
\end{pmatrix}\in\mb R^{(H-r-1)\times(d_2-r-1)},
\]
where $\bm u_2\in\mb R^{d_1-r}$ is the first column of $U_4^{(2)}$,
$v_{22}$ is the $(1,1)$-entry of $\widetilde V_4^{(2)}$,
$\bm v_{23}\in\mb R^{1\times(d_2-r-2)}$ is the remainder of the first row,
$\bm v_{32}\in\mb R^{(H-r-2)\times 1}$ is the remainder of the first column,
and $V_4^{(3)}\in\mb R^{(H-r-2)\times(d_2-r-2)}$ is the lower-right block.

On this chart we take
\[
v_{22} = a_2,\quad a_2\in[-1,1],
\]
and perform the blow-up by scaling all other entries in 
$\{U^{(2)}_1,U^{(2)}_2, V^{(2)}_3 , \widehat {\bm u}_1, \widetilde V_4^{(2)}\}$:
\begin{align}
U^{(2)}_1 &= a_1\,\widehat U^{(2)}_1, &
U^{(2)}_2 &= a_1\,\widehat U^{(2)}_2, &
V^{(2)}_3 &= a_1\,\widehat V^{(2)}_3, &
\widehat {\bm u}_1 &= a_2\,\widetilde {\bm u}_1, \nonumber\\
{\bm v}_{23} &= a_1\, \widehat {\bm v}_{23}, &
{\bm v}_{32} &= a_1\,\widehat {\bm v}_{32}, &
V_4^{(2)}   &= a_1\,\widehat V_4^{(2)}. 
\label{eq:blow-up-V4-1}
\end{align}

Jacobian of this step is $a_2^{h(1)-1}$. To make the product term explicit, note that
\[
U_4^{(2)} \widetilde V_4^{(2)}
=
\bigl(\bm u_2,\; U_4^{(3)}\bigr)
\begin{pmatrix}
a_2 & a_2\widehat {\bm v}_{23}\\
a_2\widehat {\bm v}_{32} & a_2\widehat V_4^{(3)}
\end{pmatrix}
=
a_1\Bigl(
\bm u_1 + U_4^{(2)}\widehat {\bm v}_{21},\;\;
\bm u_1\widehat {\bm v}_{12} + U_4^{(2)}\widehat V_4^{(2)}
\Bigr).
\]
Let
\[
\widehat{\bm u}_2 :=
\bm u_2 + U_4^{(3)}\widehat {\bm v}_{32},
\qquad
\widetilde V_4^{(3)} := \widehat V_4^{(3)} - \widehat {\bm v}_{32}\widehat {\bm v}_{23}.
\]
The Jacobian of this change of variables is equal to one, and in the new coordinates
the product reduces to
\[
U_4^{(2)} \widetilde V_4^{(2)}
=
a_1\Bigl(\widehat{\bm u}_2,\; \widehat{\bm u}_2\widehat{\bm v}_{23} + U_4^{(3)}\widetilde V_4^{(3)}\Bigr).
\]

We now address the domain constraints and apply Lemma~\ref{lem:finite_operator} to
control the mixed terms.
Assume that
\[
\|U_4^{(3)}\|_\infty\le \frac{1}{H-r} < \frac{1}{H-r-1},
\qquad
\|\widehat{\bm v}_{23}\|_\infty \le \frac{1}{2},
\]
so that
\[
\|U_4^{(3)}\widehat{\bm v}_{32}\|_\infty\le\frac{1}{2},
\qquad
\|\widehat{\bm v}_{32}\widehat{\bm v}_{23}\|_\infty\le \frac{1}{2}.
\]
Since all other coordinates remain in $[-1,1]$, these bounds imply that the current chart contains a hyperrectangle on which
\begin{align}
    \|U_4^{(3)}\|_\infty\le \frac{1}{H-r} < \frac{1}{H-r-1},
\quad
\|\widehat{\bm v}_{23}\|_\infty \le \frac{1}{2},\quad \|\widehat{\bm u}_2\|_\infty \le \frac{1}{2},
\quad
\|\widetilde V_4^{(3)}\|_\infty \le \frac{1}{2},\label{eq:domain 3}
\end{align}
while the remaining variables range in $[-1,1]$. Moreover, since $\|\widehat{\bm u}_1\widehat{\bm v}_{12}\|_\infty\le 1/2$, its Frobenius norm
satisfies
\[
\|\widehat{\bm u}_2\widehat{\bm v}_{23}\|_F^2
\;\le\;
\frac{(d_1-r)(d_2-r-2)}{4}.
\]
Therefore, by Lemma~\ref{lem:finite_operator},
\[
\|U_4^{(2)} \widetilde V_4^{(2)}\|_F^2
\;\le\;
a_2^2\Bigl(\tfrac{(d_1-r)(d_2-r-2)}{2}+1\Bigr)
\bigl(\|\widehat{\bm u}_2\|_F^2 + \|U_4^{(3)}\widetilde V_4^{(3)}\|_F^2\bigr) \;\le\; d_1d_2 \bigl(\|\widehat{\bm u}_2\|_F^2 + \|U_4^{(3)}\widetilde V_4^{(3)}\|_F^2\bigr).
\]
Therefore, based on \eqref{eq:psi2} we obtain, on this local chart,
\begin{align}
    \Psi & \le
    a_1^2 d_1d_2
    \bigl(\|U^{(2)}_1\|_F^2 + \|U^{(2)}_2\|_F^2 + \|V^{(2)}_3\|_F^2
          + \|\widehat{\bm u}_1\|_F^2 + \|U_4^{(2)}\widetilde V_4^{(2)}\|_F^2\bigr)\nonumber\\
          & \le
    a_1^2 a_2^2\, (d_1d_2)^2
    \bigl(\|\widehat U^{(2)}_1\|_F^2 + \|\widehat U^{(2)}_2\|_F^2 + \|\widehat V^{(2)}_3\|_F^2
    + \|\widetilde{\bm u}_1\|_F^2 + \|\widehat{\bm u}_2\|_F^2 + \|U_4^{(3)}\widetilde V_4^{(3)}\|_F^2\bigr).
    \label{eq:psi4}
\end{align}

Next, we rescale the active coordinates restricted as in \eqref{eq:domain 3} so that each ranges over $[-1,1]$ and denote the new chart as $\m S_2$. This
change of variables contributes the Jacobian factor
\begin{align}
    (H-r)^{-(d_1-r)(H-r-2)}\cdot 2^{-d_1+r - (H-r-1)(d_2-r-2)}.
    \label{eq:rescale3}
\end{align}
On the right-hand side of \eqref{eq:psi2}, the squared norm $\|\widehat{\bm u}_2\|_F^2$ is now a regular quadratic form as well. Thus we are again in the same situation as in \eqref{eq:psi2}: $\Psi$ is not yet in normal crossing form, and we have two possible choices for the next blow-up, both centered on the submanifold where all entries of $\widehat U^{(2)}_1, \widehat U^{(2)}_2, \widehat V^{(2)}_3, \widetilde{\bm u}_1, \widehat{\bm u}_2$, and $\widetilde V_4^{(3)}$ vanish. The first option is to perform a blow-up in one of the regular directions, by taking any single entry of $\widehat U^{(2)}_1, \widehat U^{(2)}_2, \widehat V^{(2)}_3, \widetilde{\bm u}_1$, or $\widehat{\bm u}_2$ as a new radial variable. The second option is to continue to blow up along the singular block $\widetilde V_4^{(3)}$.

\paragraph{Recursive structure of the blow-up and derivation of the RLCT.} The behavior of these two branches follows a simple recursive pattern. Each time we
perform a blow-up along the $\widehat V_4$–direction, one additional column $\bm u_t$ of $U_4$
is separated out and incorporated into the regular part of the parametrization. Let
$t\ge 0$ denote the number of blow-ups already performed in the $\widehat V_4$–direction and $\m S_t$ the current local chart.
Let $h(t)$, defined as in \eqref{eq:h(t)}, be the dimension of the current center
submanifold after these $t$ steps. In the $(t+1)$-th step we have the following
dichotomy.

\begin{itemize}
\item[(i)] \textbf{Regular blow-up.}
If we choose a regular direction and take one entry of the current blocks
$\widehat U^{(t)}_1, \widehat U^{(t)}_2, \widehat V^{(t)}_3, \widetilde{\bm u}_1\dots \widehat{\bm u}_t$ as a new radial variable
$a_{t+1}$, then the corresponding change of variables contributes a Jacobian factor
$|a_{t+1}|^{\,h(t)-1}$. In the new coordinates $\m S_t'$, $\Psi$ acquires an extra
multiplicative constant bounded by $d_1d_2(H-r)^2$, while the quadratic form in the
remaining variables is smooth and bounded away from zero on the unit cube. Hence
$\Psi$ is already in normal crossing form in the $(a_1,\dots,a_t,a_{t+1})$-chart,
and this branch produces a candidate pole
\[
\lambda_{a_{t+1}}= \frac{h(t)}{2}.
\]

\item[(ii)] \textbf{Singular blow-up along $\widehat V_4$.}
If instead we perform the $(t+1)$-th blow-up in the $\widehat V_4$-block, taking one entry of the current lower-right block $\widetilde V_4^{(t+1)}$ as a new radial variable $b_{t+1}$, then the Jacobian again contributes a factor $|b_{t+1}|^{\,h(t)-1}$. As in the case $t=0$ and $1$, this blow-up separates one additional column $\bm u_{t+1}$ from $U_4$, and, after imposing suitable constraints on the entries and applying Lemma~\ref{lem:finite_operator}, the new expression of $\Psi$ gains an extra global factor bounded by $d_1d_2$. A subsequent linear rescaling that maps the restricted hyperrectangle to the unit cube $[-1,1]^N$ introduces an additional Jacobian factor
\begin{align}
    (H-r)^{-(d_1-r)(H-r-t-1)}\cdot
2^{-d_1+r - (H-r-t)(d_2-r-t-1)},\label{eq:rescale4}
\end{align}
bounded from below by the expression in \eqref{eq:rescale2}, which represents when $t=0$.
In this chart $\m S_{t+1}$, the quadratic form in the remaining variables is again smooth and
bounded away from zero, and the corresponding candidate pole is
\[
\lambda_{b_{t+1}} = \frac{h(t)}{2}.
\]
\end{itemize}

\begin{figure}[t]
\centering
\footnotesize
\begin{tikzpicture}[>=stealth, node distance=2.2cm]

\node (S0) {$\m S$ along $\{\widehat U_1,\widehat U_2,\widehat V_3,\widehat V_4=0\}$};

\node[above=1.0cm of S0] (L0) {$\m S':\lambda_0 = h(0)/2$};
\draw[->] (S0) -- node[left] {\scriptsize reg.} (L0);

\node[right=1.75cm of S0] (S1) {$\mathcal S_1$};
\draw[->] (S0) -- node[below] {\scriptsize $\widehat V_4$} (S1);

\node[above=1.0cm of S1] (L1) {$\m S_1':\lambda_1 = h(1)/2\qquad$};
\draw[->] (S1) -- node[left] {\scriptsize reg.} (L1);

\node[right=1.75cm of S1] (S2) {$\mathcal S_2$};
\draw[->] (S1) -- node[below] {\scriptsize $\widehat V_4^{(2)}$} (S2);

\node[above=1.0cm of S2] (L2) {$\qquad\m S_2':\lambda_2 = h(2)/2$};
\draw[->] (S2) -- node[left] {\scriptsize reg.} (L2);

\node[right=1.75cm of S2] (Sdots) {$\cdots$};
\draw[->] (S2) -- node[below] {\scriptsize $\widehat V_4^{(3)}$} (Sdots);

\node[right=1.75cm of Sdots] (Slast) {$\mathcal S_{H-r}$};
\draw[->] (Sdots) -- node[below] {\scriptsize $\widehat V_4^{(H-r)}$} (Slast);

\node[above=1.0cm of Slast] (Llast) {$\m S_{H-r}':\lambda_{H-r} = h(H-r)/2\quad$};
\draw[->] (Slast) -- node[left] {\scriptsize reg.} (Llast);

\end{tikzpicture}
\caption{Recursive blow-up scheme and candidate poles $\lambda_t = h(t)/2$.}\label{fig:recuresive}
\end{figure}
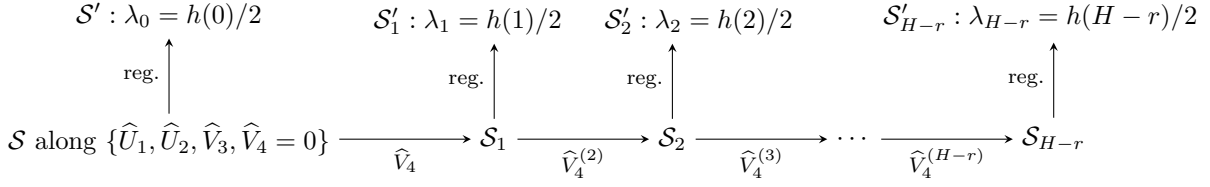

In both cases, the exponent of the new radial variable is determined by $h(t)$, so the value of the candidate pole at step $t+1$ is always $\lambda_{t+1} = h(t)/2$, while the prefactors differ by tracked constants and the explicit rescaling Jacobians. The recursive structure of the blow-up process is illustrated in Figure~\ref{fig:recuresive}. Consequently, the RLCT is given by
\[
\lambda \;=\; \frac{1}{2}\cdot\min_{0\le t\le H-r} h(t),
\]
and the multiplicity $m$ is equal to the number of integers $t\in\{0,\dots,H-r\}$ at which this minimum is attained. Expand the expression of $h(t)$ defined in \eqref{eq:h(t)} gives,
\begin{align*}
h(t) & = r(d_1+d_2-r) + t(d_1-r) + (H-r-t)(d_2-r-t)\\
&=
t^2 + (d_1-d_2-H+r)\,t
+ \Bigl[r(d_1+d_2-r)+(H-r)(d_2-r)\Bigr].
\end{align*}
Thus $h(t)$ is a convex quadratic in $t$ with vertex at
\begin{equation}\label{eq:tstar}
t^\star
=
-\frac{d_1-d_2-H+r}{2}
=
\frac{H-r+d_2-d_1}{2}.
\end{equation}

To determine the RLCT, we compare this vertex $t^\star$ with the admissible range $[0,H-r]$ of $t$ and translate the conditions into explicit inequalities involving $(d_1,d_2,H,r)$.

\noindent\textbf{Case 1: boundary minimum at $t=0$.}
The condition $t^\star\le 0$ is equivalent to
\[
\frac{H-r+d_2-d_1}{2} \;\le\; 0
\quad\Longleftrightarrow\quad
H \;\le\; d_1 - d_2 + r.
\]
In this regime $h(t)$ is increasing with $t$, so the minimum is attained at
$t=0$. Since \(h(0) = Hd_2 - Hr + d_1r\)
we obtain
\[
\lambda = \frac{h(0)}{2} = \frac{Hd_2 - Hr + d_1r}{2},\qquad m=1.
\]

\medskip
\noindent\textbf{Case 2: boundary minimum at $t=H-r$.}
The condition $t^\star\ge H-r$ is equivalent to
\[
\frac{H-r+d_2-d_1}{2} \;\ge\; H-r
\quad\Longleftrightarrow\quad
H \;\le\; d_2 - d_1 + r.
\]
In this regime $h(t)$ is decreasing on $[0,H-r]$, so the minimum is attained at
$t=H-r$. A direct substitution gives \(h(H-r) = Hd_1 - Hr + d_2r,\)
hence
\[
\lambda = \frac{h(H-r)}{2} = \frac{Hd_1 - Hr + d_2r}{2}, \qquad m=1.
\]

\medskip
\noindent\textbf{Case 3: interior minimum.}
If
\(
H \ge r + |d_2 - d_1|,
\)
then $0<t^\star<H-r$ and the continuous minimum of $h(t)$ is attained at $t=t^\star$. On the discrete set of integers $t\in\{0,\dots,H-r\}$, the minimum is achieved at the integer point(s) nearest to $t^\star$, and its value depends on the parity of \(2t^\star = H - d_1 + d_2 - r\).

\begin{itemize}
\item If $2t^\star$ is even, then $t^\star$ is an integer and the discrete minimum coincides
with the continuous minimum. In this case
\[
\lambda
=
\frac{1}{2}h(t^\star)
=
\frac{Hd_2 - Hr + d_1r}{2}
-
\frac{(H - d_1 + d_2 - r)^2}{8},
\qquad
m=1.
\]

\item If $2t^\star$ is odd, then $t^\star$ is a half-integer, and the minimum over integer $t$ is attained at the two adjacent points $t=t^\star\pm\frac12$, which are symmetric around the vertex. Since $h(t)$ has unit leading coefficient, we have \(\min_{t\in\mb Z} h(t) = h\left(t^\star+\tfrac12\right) = h(t^\star)+\frac14,\)
which yields
\[
\lambda
=
\frac{1}{2}\Bigl(h(t^\star)+\frac14\Bigr)
=
\frac{Hd_2 - Hr + d_1r}{2}
-
\frac{(H - d_1 + d_2 - r)^2 + 1}{8},
\qquad
m=2.
\]
\end{itemize}

\paragraph{Step 3.}
In the final step we explicitly calculate the integral
\[
-\log \int_\Theta \exp\Big\{-\tfrac{(1+\omega L/2)\,\omega n}{2}\, R(M,M^\star)\Big\}\varphi(d\theta)
\]
under the new parameterization after resolution of singularities. Since we have already obtained the expressions for $\lambda$ and $m$, the final step is to explicitly derive the constant expression in the upper bound of the excess risk. According to the analysis in Step 2, the fully resolved chart $\m S_{H-r}'$ on the right end of Figure~\ref{fig:recuresive}, where we have performed $H-r$ blow-ups along $\widehat V_4$ and the another blow-up along a regular directions, must contain all the pole candidates. Therefore, we focus on the integral within this local coordinate chart.

On which, by Lemma~\ref{lem:canonical_M*},
\begin{align*}
    \|UV - M^\star\|_F^2\le \sigma^2_{\max}(P_0) \sigma^2_{\max}(Q_0)\left\|U'V' - M_r\right\|_F^2.
\end{align*}
By \eqref{eq:psi}
\begin{align*}
    \|U'V'-M_r\|_F^2 \le 3\Big(\|\widehat U_1\|_F^2 + \|\widehat U_2\|_F^2 + \|\widehat V_3\|_F^2  + \|U_4\widehat V_4\|_F^2 \Big) =: 3\cdot\Psi.
\end{align*}
By \eqref{eq:psi multiple1} and the analysis of the recursive structure of the blow-up,
\begin{align*}
    \Psi & \le
    a_1^2 d_1d_2
    \bigl(\|U^{(2)}_1\|_F^2 + \|U^{(2)}_2\|_F^2 + \|V^{(2)}_3\|_F^2
          + \|\widehat{\bm u}_1\|_F^2 + \|U_4^{(2)}\widetilde V_4^{(2)}\|_F^2\bigr) 
          \\ &\le
    a_1^2 a_2^2\, (d_1d_2)^2
    \bigl(\|\widehat U^{(2)}_1\|_F^2 + \|\widehat U^{(2)}_2\|_F^2 + \|\widehat V^{(2)}_3\|_F^2
    + \|\widetilde{\bm u}_1\|_F^2 + \|\widehat{\bm u}_2\|_F^2 + \|U_4^{(3)}\widetilde V_4^{(3)}\|_F^2\bigr)
    \\ & \le \dots \le a_1^2a_2^2\cdots a_{H-r}^2 b_{H-r+1}^2 \, (d_1d_2)^{H-r+1}(H-r)^2 \\
    & = :\Psi'.
\end{align*}

Denote the complete resolution map from $\m S'_{H-r}$ to a subset of the original parameter space by $\theta = g(\bm a)$, where vector $\bm a$ denotes the collection of parameters on $\m S_{H-r}'$. $\m S_{H-r}'$ is a demorphism to a cube with each side in $[-1, 1]$. We focus on the subset of which each coordinate is positive, the cube of $[0, 1]^N$, where $N = (d_1+d_2)H$ is the total number of entries in $U,V$. Then all the Jacobian determinants and rescaling factors gathered as, by Lemma~\ref{lem:canonical_M*}, equation~\eqref{eq:Jacobian1}, \eqref{eq:rescale1} and \eqref{eq:rescale4}, 
\begin{align*}
    |J_g| &  = C_J\cdot\prod_{t=0}^{H-r-1}(a_{t+1})^{\,h(t)-1}\cdot(b_{H-r+1})^{h(H-r)-1},
\end{align*}
where
\begin{align*}
    C_J \ge \varphi_0& \cdot 2^{-r|d_2-d_1-r|}\cdot r^{-r^2} \cdot(2r(d_2-r))^{-r(d_2-r)} \cdot (2r(H-r)(d_2-r))^{-r(H-r)} \\
    & \cdot \Big((H-r)^{-(d_1-r)(H-r-1)}\cdot 2^{-d_1+r - (H-r)(d_2-r-1)}\Big)^{H-r}\cdot |\det(P_0)|^{-H}|\det(Q_0)|^{-H}.
\end{align*}
We observe that except for the last factor, the negative logarithm of all the others is no greater than $\max\{d_1,d_2\}(H-r)^2\log(2rd_1d_2)$, therefore,
\begin{align*}
    -\log C_J \le H \log\big(|\det(P_0)\cdot\det(Q_0)|\big) + 6\max\{d_1,d_2\}(H-r)^2 \log (2rd_1d_2).
\end{align*}
Moreover, 
\begin{align*}
    & -\log \int_\Theta \exp\Big\{-\tfrac{(1+\omega L/2)\,\omega n}{2}\, R(M,M^\star)\Big\}\varphi(d\theta) \\
    & \qquad \le -\log\int_{[0,1]^N} \exp\Big\{-C_{\Psi'}\cdot\omega \, n\, \Psi' \Big\} \prod_{t=0}^{H-r-1}(a_{t+1})^{\,h(t)-1}\cdot(b_{H-r+1})^{h(H-r)-1} d \bm a - \log (\varphi_0 C_J),
\end{align*}
where $C_{\Psi'} = \big(1+\frac{\omega L}{2}\big) ({3\sigma^2_{\max}(P_0) \sigma^2_{\max}(Q_0)})/({d_1d_2})$. Precisely,

\[
C_{\Psi'}\cdot\Psi' = \big(1+\tfrac{\omega L}{2}\big) \cdot{3\sigma^2_{\max}(P_0) \sigma^2_{\max}(Q_0)}\cdot a_1^2a_2^2\cdots a_{H-r}^2 b_{H-r+1}^2 \, (d_1d_2)^{H-r}(H-r)^2. 
\]
Since $3\bigl(1+\tfrac{\omega L}{2}\bigr)(d_1d_2)^{H-r}(H-r)^2 \ge 1$,
Lemma~\ref{lemma:thm 4.7} applies with the choice
\[
\beta = \omega\,\sigma^2_{\max}(P_0)\sigma^2_{\max}(Q_0),
\qquad
n' = 3\Bigl(1+\tfrac{\omega L}{2}\Bigr)(d_1d_2)^{H-r}(H-r)^2.
\]
Moreover, in the present setting we have $m\le 2$, $k_1 = k_2 = \cdots = k_m = 1$, and
$h(t)$ is uniformly bounded as in \eqref{eq:h(t) upper}. Hence Lemma~\ref{lemma:thm 4.7} yields
\begin{align*}
    & -\log\int_{[0,1]^N} \exp\Big\{-\beta\,n'a_{1:H-r}^2b_{H-r+1}^2 \Big\} \prod_{t=0}^{H-r-1}(a_{t+1})^{\,h(t)-1}\cdot(b_{H-r+1})^{h(H-r)-1} d \bm a\\
    &\qquad \le \lambda \log n' - (m-1)\log\log n' + (H-r+1)\log(H(d_1+d_2-r)) + 2\log 2 + \log\lambda +\beta\\
    & \qquad \le \lambda\log n -(m-1)\log\log n + \lambda\bigg[\log\Big(3\big(1+\tfrac{\omega L}{2}\big)\Big) + (H-r)\log(d_1d_2) + 2\log (H-r)\bigg] \\
    & \qquad \quad + (H-r+1)\log(H(d_1+d_2-r)) + 2\log 2 + \log\lambda +\beta.
\end{align*}
A simple scaling argument shows that the $n$-independent terms in the above expression are bounded from above by
\[
3H(d_1+d_2-r)(H-r+2)\log (d_1d_2)
+ \frac{1}{2} H(d_1+d_2-r)
  \log\Bigl(3\bigl(1+\tfrac{\omega L}{2}\bigr)\Bigr)
+ \beta,
\]
since $\lambda\le\frac12\max_t h(t)\le \frac{1}{2}H(d_1+d_2-r)$.
Collecting all contributions, we obtain
\begin{align*}
    -\log \int_\Theta
    \exp\Big\{-\bigl(1+\tfrac{\omega L}{2}\bigr) \omega n\, R(M,M^\star)\Big\}\,\varphi(d\theta)
    \;\le\;
    \lambda\log n - (m-1)\log\log n + C_1,
\end{align*}
where
\begin{align*}
    C_1
    &= 9(H-r+2)^2(d_1+d_2-r)\log(2rd_1d_2)
    + H \log|\det(P_0)\det(Q_0)|
    \\
    &\quad
    + \frac{H(d_1+d_2-r)}{2}\log\Bigl(3+\tfrac{3\omega L}{2}\Bigr)
    + \omega\,\sigma^2_{\max}(P_0) \sigma^2_{\max}(Q_0).
\end{align*}
Therefore,
\[
C_1
= \mathcal{O}\Bigl((H-r+2)^2(d_1+d_2-r)\log\bigl(2rd_1d_2(1+L)\bigr)\Bigr)
  + C(\omega,P_0,Q_0),
\]
where
\[
C(\omega,P_0,Q_0)
= H \log|\det(P_0)\det(Q_0)|
  + \omega\,\sigma^2_{\max}(P_0) \sigma^2_{\max}(Q_0)
\]
depends only on $(\omega,P_0,Q_0)$.

\end{document}